\documentclass[11pt]{article}

\usepackage[margin=1in]{geometry}

\usepackage[numbers,sort&compress]{natbib}

\usepackage[utf8]{inputenc} 
\usepackage[T1]{fontenc}    
\usepackage{hyperref}       
\usepackage{url}            
\usepackage{booktabs}       
\usepackage{amsfonts}       
\usepackage{nicefrac}       
\usepackage{microtype}      
\usepackage{xcolor}         
\usepackage{amsmath}
\usepackage{amssymb}
\usepackage{mathtools}
\usepackage{amsthm}
\usepackage{sansmath}
\usepackage{algorithm}
\usepackage{algorithmic}
\usepackage[skip=0.333\baselineskip]{caption}
\usepackage{subcaption}
\usepackage{graphicx}

\theoremstyle{plain}
\newtheorem{theorem}{Theorem}

\newtheorem{lemma}[theorem]{Lemma}

\theoremstyle{definition}

\theoremstyle{remark}

\newtheorem*{theorem*}{Theorem}

\newcommand{\RR}{\mathbb{R}}

\newcommand{\eps}{\varepsilon}

\newcommand{\EE}{\mathbb{E}}
\newcommand{\PR}{\mathbb{P}}

\newcommand{\step}{\mathrm{sign}}
\newcommand{\AM}{\mathrm{AM}}
\newcommand{\GM}{\mathrm{GM}}
\newcommand{\HM}{\mathrm{HM}}

\newcommand{\OPT}{\mathsf{OPT}}
\newcommand{\SDP}{\mathsf{\widetilde{Bias}}}
\newcommand{\RND}{\mathsf{ROUND}}

\newcommand{\B}{\mathsf{Bias}}

\title{Conditioned free-energy density of proteins using unbalanced solutions to constraint satisfaction problems}

\author{Pratik Worah\thanks{Corresponding author email: pworah@cims.nyu.edu},~~
 Subhash Khot,~and~ Srinivasa Varadhan}
\begin{document}

\maketitle

\begin{abstract}
We show that computing the log-partition function (free-energy) of conditioned inhomogeneous Curie--Weiss spin Hamiltonians reduces to an unbalanced $2 \to 1$ norm computation, and design a polynomial-time SDP algorithm for this problem with a lower bound proof for the amount of unbalance achieved. Applied to the protein Ubiquitin, the framework starts from a known crystal structure, explores alternative backbone conformations across the free-energy landscape, and identifies flexible regions of the protein while preserving its native secondary structure.
\end{abstract}

\section{Introduction}
Constraint Satisfaction Problems (CSPs) arise in many settings ranging from the  abstract (boolean satisfiability problems like $3$-SAT~\cite{cook}) to the practical (protein structure prediction based on CSPs~\cite{hartnew}). One common theme among the solutions of the underlying CSPs is that combinatorics dictates that most solutions will have high entropy and thus low information. More concretely, for the boolean case, if the solution space is $\{0,1\}^n$ then most satisfiable CSPs will have solutions that have roughly equal numbers of $0$s and $1$s. Mathematically, there's nothing wrong with such "balanced" solutions. However, such solutions do not contain a lot of information in practice, and implicitly (if not explicitly) one often seeks unbalanced solutions (if they exist).

For example, proteins fluctuate among an ensemble of backbone conformations at physiological temperature. The experimentally observed structure represents the free-energy minimum: the state that optimally balances low energy with conformational entropy~\cite{pathria2011}. Accessing conformational states beyond the native fold---transition states, excited states, and partially unfolded intermediates---is essential for understanding protein function, allosteric regulation, and rational drug design~\cite{hartnew}, yet existing computational tools are largely confined to predicting the ground-state structure. Existing works like~\cite{alphafold2, rfold} are effective for ground-state prediction of structure from the amino acid sequence but do not directly provide a thermodynamic framework for alternative conformational states. Molecular dynamics with force fields such as OPLS4~\cite{opls4} can in principle sample excited states, but faces exponential mixing times for frustrated spin systems with mixed-sign couplings~\cite{barahona1982,jerrum1993}. NMR-specific tools such as TALOS+~\cite{talos} and CS-Rosetta~\cite{csrosetta} predict dihedral angles from chemical shifts but are limited to ground-state structures. Rosetta-based flexible backbone methods~\cite{rosetta} lack a continuous parameterization of the free-energy landscape. On the theoretical side, Jerrum and Sinclair~\cite{jerrum1993} give a fully polynomial-time approximation scheme for ferromagnetic Hamiltonians, and unit-rank Hamiltonians admit closed-form solutions~\cite{cw1}; neither assumption holds for the mixed-sign, high-rank couplings arising from NMR-derived protein Hamiltonians.

A common thread links these limitations: computing the free-energy of a frustrated spin system is computationally hard in general, and existing methods either sidestep the problem (by targeting only the ground state) or lack polynomial-time guarantees. Our approach is to reformulate the problem in a way that admits efficient convex optimization. Specifically, we show that the log-partition function (free-energy) of inhomogeneous Curie--Weiss Hamiltonians is asymptotically equivalent to a constrained $\ell_1$-maximization problem---the unbalanced $2 \to 1$ operator norm of the coupling matrix. This variational characterization reduces the physical problem of computing conditioned energy states to a semidefinite program. The unbalanced $\ell_1$-maximization problem also arises independently in sign-constrained regression~\cite{meinshausen2013,nnls}, clinical risk scoring~\cite{ustunrudin2016}, and portfolio optimization~\cite{konnoyamazaki1991}, so our algorithmic results may be of interest in those domains as well.

Building on this characterization, we design an SDP-based algorithm that, for any conditioning parameter $\varepsilon \in (0,1]$, computes an $\varepsilon$-approximate unbalanced solution with provable guarantees in polynomial time. We demonstrate the framework on the 76-residue protein Ubiquitin (PDB: 1UBQ), constructing the spin Hamiltonian from J-coupling constants derived via the Karplus equation from the crystal structure. By sweeping $\varepsilon$, we systematically explore the conditioned conformational landscape, identifying structurally labile backbone regions while preserving the native secondary structure.

In summary, this paper makes two contributions: (1)~a variational characterization of the Curie--Weiss free-energy as an unbalanced $2 \to 1$ norm computation, together with an SDP algorithm with provable approximation guarantees (Theorems~\ref{thm:var-char},~\ref{main:thm}, and~\ref{thm:polytime}); and (2)~a computational pipeline that applies this framework to explore conditioned protein conformational landscapes starting from known crystal structures. Section~\ref{sec:theory} presents the theoretical framework, Section~\ref{sec:application} applies it to Ubiquitin, and Section~\ref{sec:discussion} discusses the results. Proofs and additional experiments appear in the supplement.

\section{Theoretical Framework}\label{sec:theory}
In this section we outline the main theoretical contributions of our paper. The formal statements and the proofs of the results described in this section are present in Supplement Sections~\ref{sbs:thm-lm-stmt} and~\ref{sbs:det-proofs}.
\subsection{Curie--Weiss Hamiltonians and free-energy}\label{sec:cw}

In the inhomogeneous Curie--Weiss model~\cite{cw1,cw2}, we have a positive definite matrix $A \in \RR^{m \times m}$ and $\pm 1$-valued spin variables $x_1, \ldots, x_n$. The Hamiltonian is:
\begin{equation}\label{eqn:ham}
    H(x) = \sum_{i\in[m]}\left(\sum_{j\in[n]} h_{ij} x_j \right)^2,
\end{equation}
where each column $h_j$ is sampled uniformly from rows of $A$. The Gibbs measure is $p(x) := \frac{1}{Z} \exp(-\frac{1}{n} H(x))$, with partition function $Z := \sum_{x \in \{\pm 1\}^n} \exp(-\frac{1}{n} H(x))$. The log-cumulant is $\Psi(m) := \lim_{n \to \infty} \frac{1}{n} \log Z$. A summary of notation is provided in Supplementary Table~\ref{tab:notation}.

The log-partition function is proportional to the negative Helmholtz free-energy: $\log Z = -\beta F$, where $F = U - TS$, with $U$ the internal energy, $T = 1/\beta$ the temperature, and $S$ the entropy~\cite{pathria2011}. The log-partition function captures the free-energy, not just the energy; the two coincide only at zero temperature ($\beta \to \infty$). For proteins at physiological temperature, the distinction is significant: Ubiquitin alone has $\sim$152 continuous dihedral degrees of freedom, and states slightly higher in energy that access many more conformations can be thermodynamically favoured.

\subsection{Variational Characterization: Free-energy as $\ell_1$-Maximization}\label{sec:varchar}

Our first theoretical result connects the log-partition function to a constrained $\ell_1$-maximization problem:

\begin{theorem}[Variational characterization]\label{thm:var-char}
For large $m, n$ with $\frac{m}{n}\to 0$, $\Psi(m)$ is asymptotically equal to:
\begin{equation}\label{Psi:eqn}
       \Psi(m)\simeq\max_{\|z\|_2^2\le y\atop y\in\RR,z\in\mathbb{R}^m}\left(\underbrace{\beta \cdot \|Az\|_1}_{\text{energy contribution}}-\underbrace{y}_{\text{entropy cost}}\right).
\end{equation}
Moreover, for a maximizer $z^*$, the magnetization equals:
\begin{equation}\label{eqn:mag-eqv}
    \Phi = \frac{1}{m}\sum_{i\in[m]}\step(\langle A_i, z^*\rangle),
\end{equation}
where $A_i$ denotes the $i^{th}$ row of $A$.
\end{theorem}
\noindent The proof is given in Supplementary Section~\ref{sec:proof-varchar}.

The term $\beta\|Az\|_1$ is the energy contribution from the Hamiltonian, while the constraint $\|z\|_2^2 \le y$ bounds the entropy cost: larger $y$ permits larger fluctuations in $z$, corresponding to higher entropy. The magnetization of the optimal spin configuration is determined by the row signs of $A$ applied to the optimizer (Equation~\ref{eqn:mag-eqv}). This equivalence reduces the problem of computing conditioned free-energy states to a convex optimization problem.

Setting $y = 0$ would eliminate the entropy term, reducing equation~\eqref{Psi:eqn} to pure energy minimization. At finite temperature this is physically incorrect.

\subsection{The Unbalanced $2\to 1$ Norm Problem}\label{sec:unbal}

The variational problem in Theorem~\ref{thm:var-char} is a constrained instance of the $2 \to 1$ operator norm. To formalize this, given matrices $A, B \in \RR^{n \times n}$, define the unbalance with respect to $B$ as:
\begin{equation}
    \delta=\frac{1}{n^2}\cdot\sum_{i,j\in[n]}\left(\frac{\step(\langle B_i,x\rangle)+\step(\langle B_j,x\rangle)}{2}\right)^2.\label{unbal:eqn}
\end{equation}
The unbalanced $2\to 1$ norm problem seeks to maximize $\|Ax\|_1$ subject to $\|x\|_2 \le 1$ and a lower bound on $\delta$. Note that $\delta\in[\frac{1}{2},1]$. Since the $2 \to 1$ norm is NP-hard to compute exactly~\cite{bhat}, we design an SDP-based algorithm that sweeps the parameter space to obtain an approximate solution. 

\subsection{SDP Algorithm and Approximation Guarantees}\label{sec:algorithm}
Define $\mathrm{SDP}(\varepsilon,\xi)$ with input parameters $\varepsilon\in(0,1)$ and $\xi\in\RR$, as follows:
\begin{eqnarray}
    \max_{U, V\in\RR^{n\times n}} & \frac{1}{n^2}\cdot\sum_{i,j\in[n]}\log\left(\|B_i\|_V^2+\|B_j\|_V^2+2\cdot\langle B_i, B_j\rangle_V\right) - \log \xi\label{sdp-rlx1}\\
    \mathrm{s.t.}\  \quad & \frac{4}{n}\cdot\sum_{i\in[n]}\|B_i\|_V^2\le \xi,\label{sdp-rlx2}\\
    \quad & \sum_{i,j\in[n]} A_{ij}\langle v_i,u_j\rangle\ge\varepsilon\cdot\OPT,\label{sdp-rlx3}\\
    \quad & \frac{1}{n}\cdot\sum_{i\in[n]}\|v_i\|^2\le 1,\label{sdp-rlx4}\\
    \quad & \forall i\in[n]:\|u_i\|=1,\label{sdp-rlx5}\\
    &\forall i, j\in[n]:\ V_{ij}=\langle v_i, v_j\rangle, \ U_{ij} = \langle u_i, u_j\rangle.\label{sdp-rlx6}
\end{eqnarray}
The SDP relaxation uses two semidefinite matrices $U, V \in \RR^{n \times n}$, with an objective based on the log-mutual-coherence of the rows of $B$ under the metric induced by $V$, subject to an $\ell_1$-objective lower bound $\varepsilon \cdot \OPT$. $\mathrm{SDP}(\varepsilon,\xi)$ is solved by Algorithm~\ref{algmain} for different parameter values to find unbalanced solutions.

The conditioning parameter $\varepsilon$ controls how close the solution must remain to the unconstrained optimum $\OPT$: large $\varepsilon$ constrains the solution near the ground-state energy, while small $\varepsilon$ permits exploration of highly unbalanced configurations.

\begin{algorithm}[htb!]
\caption{Randomized rounding for computing unbalanced solutions of $2\to 1$ norms}\label{algmain}
\begin{algorithmic}[1]
\STATE {\footnotesize {\bfseries Input:} Matrices $A, B$ that define the unbalanced $2\to 1$ norm problem, and parameter $\varepsilon$.}
\STATE {\footnotesize {\bfseries Output:} Unbalance factor $\SDP$, solution matrix $\tilde{V}$, and rounded solution $\{\tilde{x}_i\}_{i\in[n]}, \{\tilde{y}_i\}_{i\in[n]}$.}
\STATE Solve unconstrained SDP relaxation to obtain $\OPT$.
\STATE $\SDP\leftarrow 0$
\STATE $h\approx\frac{1}{n^{c^*}}$\quad\# See Supplementary Theorem~\ref{thm:polytime}
\FOR {$\xi\in \{h,h+\frac{h}{n},h+\frac{2h}{n},...,n^3\cdot (\max_{i,j}|B_{i,j}|)^2\}\}$}
\IF{$\mathrm{SDP}(\varepsilon,\xi)$ not feasible}
\STATE continue;
\ENDIF
\STATE Set $\SDP(\varepsilon,\xi), V(\varepsilon,\xi), \{x_i\},\{y_i\}\leftarrow\RND(\mathrm{SDP}(\varepsilon, \xi))$
\IF {$\SDP(\varepsilon,\xi) > \SDP$}
\STATE $\SDP\leftarrow\SDP(\varepsilon,\xi),\widetilde{V}\leftarrow V(\varepsilon,\xi), \{\tilde{x}_i\}\leftarrow\{x_i\},\{\tilde{y}_i\}\leftarrow\{y_i\}$
\ENDIF
\ENDFOR
\RETURN $\SDP, \widetilde{V}, \{\tilde{x}_i\}_{i\in[n]}, \{\tilde{y}_i\}_{i\in[n]}$
\end{algorithmic}
\end{algorithm}

\begin{algorithm}[htb]
\caption{$\RND:\RR^n\mapsto\RR$ (rounding procedure)}\label{alground}
\begin{algorithmic}[1]
\STATE {\footnotesize {\bfseries Input:} $\mathrm{SDP}(\varepsilon,\xi)$ solution matrices $U(\varepsilon,\xi),V(\varepsilon,\xi)$.}
\STATE {\footnotesize {\bfseries Output:} Rounded solutions $\{x_i,y_i\}_{i\in[n]}$.}
\STATE Sample standard Gaussian vector $g=(g_1,...,g_n)$.
\STATE Let $\{v_i\},\{u_i\}$ be the Cholesky decompositions of $V,U$ respectively.
\STATE $\forall i\in[n]:$ set $x_i\leftarrow \langle v_i, g \rangle$
\STATE $\forall i\in[n]:$ set $y_i\leftarrow \step(\langle u_i, g \rangle)$
\RETURN $\SDP(\varepsilon,\xi), V(\varepsilon,\xi),\{x_i\}_{i\in[n]},\{y_i\}_{i\in[n]}$
\end{algorithmic}
\end{algorithm}
\paragraph{Lower bound on unbalance.} Our second theoretical result provides lower bounds on the unbalance achieved by Algorithm~\ref{algmain}, and also show that it runs in polynomial time.
\begin{theorem}[Informal main guarantee]\label{main:thm}
For any $\varepsilon \in (0,1]$, Algorithm~\ref{algmain} computes a $\SDP$-unbalanced solution with expected rounded objective at least $\Omega(\varepsilon \cdot \OPT)$ and unbalance factor bounded by the AM-GM-HM ratios of the matrix $V^*$ norms (see Theorem~\ref{act-main:thm} for the precise bound).
\end{theorem}

\begin{theorem}[Polynomial time]\label{thm:polytime}
Under mild assumptions on the coupling matrix (bounded entries, positive definite optimal $V^*$), Algorithm~\ref{algmain} runs in polynomial time and produces a solution within a $(1 + \frac{1}{n})$ factor of the optimum.
\end{theorem}

\noindent Proofs are given in Supplementary Section~\ref{sec:proof-main}.

\paragraph{Empirical validation on Gaussian Orthogonal Ensemble.} SDP rounding techniques for boolean optimization~\cite{gw,alon-naor,khot-naor} extend to real-valued $p \to q$ norms~\cite{bhatarx,bhatarx2,bhat}; the connection to $\ell_1$-PCA in the unit-rank case is established in~\cite{l1pca1}. We validated Algorithm~\ref{algmain} on synthetic $\ell_1$-PCA instances with Gaussian Orthogonal Ensemble matrices (Supplementary Fig.~\ref{fig:l1pca_quad}). The relative unbalance factor ranges from $1.8$ at small $\varepsilon$ to $1.3$ at $\varepsilon \to 1$, and the objective ratio varies linearly with $\varepsilon$, confirming that the SDP constraint remains tight.

\section{Application: Conditioned Protein Conformational Landscapes}\label{sec:application}

\subsection{Pipeline: From PDB Structure to Conditioned Conformations}\label{sec:pipeline}

We apply the framework of Section~\ref{sec:theory} to the small regulatory protein Ubiquitin (76 residues, PDB: 1UBQ~\cite{vijaykumar1987}). We emphasize that the results here are proof of concept and we discuss the limitations towards the end of the paper. 

The pipeline from crystallographic structure to conditioned Ramachandran plots is summarized in Algorithm~\ref{algpdb}. The key steps are: (i) compute J-coupling constants from PDB backbone dihedral angles via the forward Karplus equation~\cite{karplus1959}; (ii) construct the spin Hamiltonian; (iii) solve the constrained SDP; and (iv) recover conditioned backbone angles via Karplus inversion.

\begin{algorithm}[htb!]
\caption{Conditioned secondary structure via forward Karplus and SDP from PDB structure data}\label{algpdb}
\begin{algorithmic}[1]
\STATE {\bfseries Input:} PDB crystal structure with backbone dihedral angles $(\phi_i, \psi_i)$ for $n$ residues; parameter $\varepsilon\in(0,1)$.
\STATE {\bfseries Output:} Conditioned backbone dihedral angles $(\phi'_i, \psi'_i)$ and Ramachandran plots.
\STATE \textbf{Forward Karplus:} For each residue $i$, compute ${}^3J_{\text{HC}}(\phi_i)$ and ${}^2J_{\text{CAN}}(\psi_i)$ from the PDB angles.
\STATE \textbf{Hamiltonian construction:} For each residue pair $(i,j)$, set $H_{ij} \leftarrow \sum_k\, J_{ij}^{(k)}$; set missing entries to $0$.
\STATE \textbf{PSD shift and SOS decomposition:} Compute $\tilde{H} \leftarrow H + |\lambda_{\min}(H)|\cdot\mathrm{I}$; factorize $\tilde{H} = AA^T$ (Cholesky).
\STATE \textbf{Unbalanced SDP:} Set $B \leftarrow A$. Run Algorithm~\ref{algmain} with parameter $\varepsilon$. Record the most unbalanced solution $\vec{s}' \leftarrow \mathrm{sign}(\tilde{x})$.
\STATE \textbf{Ground-state approximation:} Solve the unconstrained SDP, round $\vec{s} \leftarrow \mathrm{sign}(X^* g)$ over 1000 draws; retain the $\vec{s}$ minimizing $\vec{s}^T H \vec{s}$.
\STATE \textbf{Sign-flip:} Construct $H_{\text{flip}}$ by negating $H_{ij}$ where $s_i s_j \ne s'_i s'_j$.
\STATE \textbf{Karplus inversion (original):} Invert ${}^3J$ for $\phi_i$; invert ${}^2J$ for $\psi_i$. Among the multiple solutions of the non-injective Karplus equation, select the branch nearest to the known (from PDB) ground-state angle.
\STATE \textbf{Karplus inversion (conditioned):} Apply blending $J_{\text{eff}} = (2\alpha - 1)J$ for flipped pairs ($\alpha = 0.9$). Invert to obtain $\phi'_i, \psi'_i$, again selecting the nearest branch.
\RETURN Ramachandran plots for original and conditioned structures.
\end{algorithmic}
\end{algorithm}

\subsection{Hamiltonian Construction and Free-Energy Landscape}\label{sec:felscape}

The backbone dihedral angles $(\phi_i, \psi_i)$ for all 76 residues are extracted from the 1UBQ crystal structure (1.8~\AA\ resolution). We compute two types of J-coupling constants: (i)~three-bond couplings ${}^3J_{\text{HC}}(\phi_i)$ for residues $i = 2, \ldots, 76$, yielding 75 values for diagonal entries $H_{ii}$; and (ii)~two-bond couplings ${}^2J_{\text{CAN}}(\psi_i)$ for residues $i = 1, \ldots, 75$, yielding 75 values for off-diagonal entries $H_{i,i+1}$ (see Methods for Karplus equations and fitted coefficients). This gives 150 PDB-derived J-coupling constants with complete backbone coverage. The Hamiltonian is shifted to positive semidefinite form $\tilde{H} = H + |\lambda_{\min}(H)| \cdot I$ and factorized as $\tilde{H} = AA^T$ via Cholesky decomposition (Supplementary Fig.~\ref{fig:pdb-ham}).

To identify the free-energy optimum, we swept the entropy budget $y$ from 1 to $20n$ and computed
\[
  F(y) \;=\; \max\bigl\{\|Az\|_1-y : \|z\|_2^2 \le y\bigr\}
\]
at each grid point. The resulting landscape (Fig.~\ref{fig:fe-landscape}) is concave, with a maximum at $y^* \approx 592$ ($\approx 7.8\,n$). The optimum lies at an interior point: restricting to $y = n = 76$ (the energy-only case, ignoring entropy) gives a free-energy roughly 40\% below the optimum, confirming that the entropy contribution is substantial.

\begin{figure}[htbp]
\centering
\includegraphics[width=0.6\textwidth]{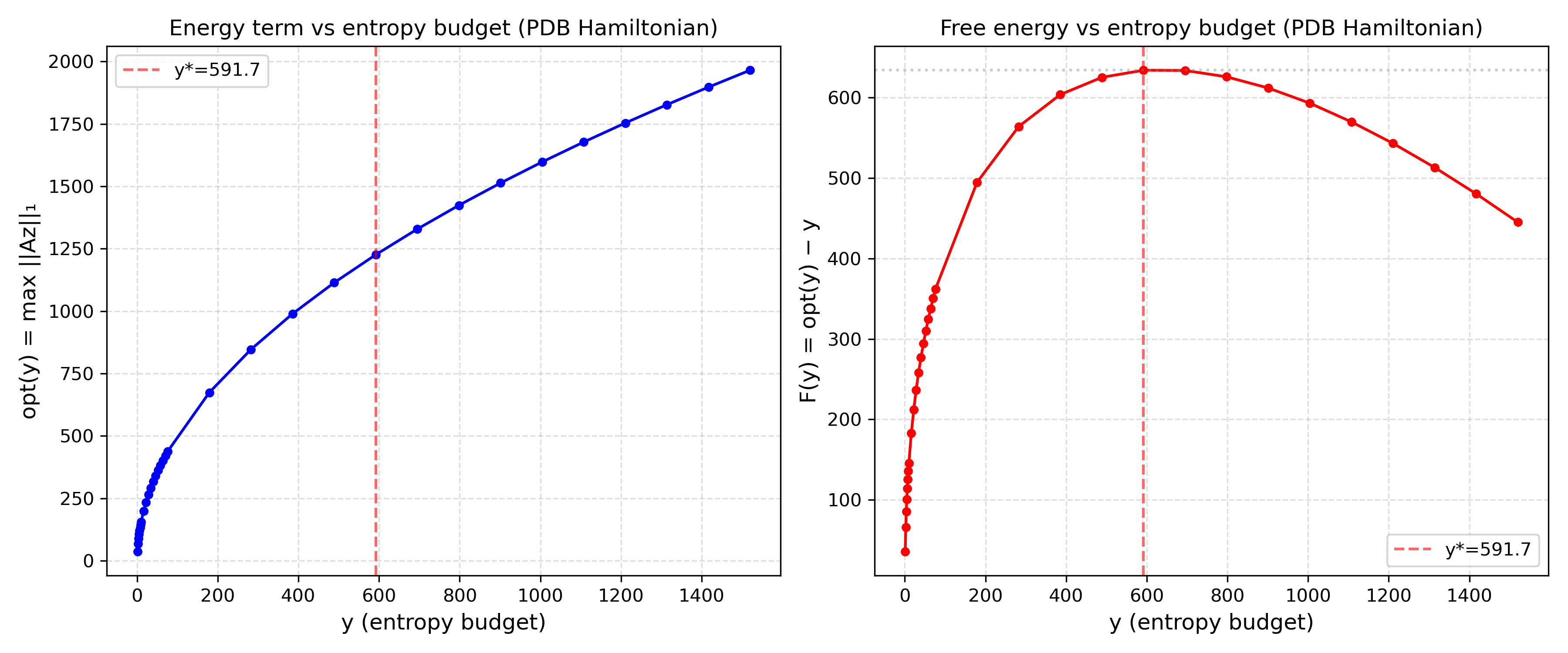}
\caption{Free-energy landscape $F(y) = \text{opt}(y) - y$ for the Ubiquitin Hamiltonian ($n = 76$). The concave shape confirms a unique free-energy optimum at $y^* \approx 592$ ($\approx 7.8\,n$), well above the energy-only value $y = n$.}
\label{fig:fe-landscape}
\end{figure}

\subsection{Conditioned Structural Ensemble via $\varepsilon$-Sweep}\label{sec:epssweep}

The PDB ground-truth Ramachandran plot for 1UBQ shows Ubiquitin's mixed $\alpha/\beta$ fold: 40 $\beta$-sheet residues, 16 $\alpha_R$-helix, 5 left-handed $\alpha$, and 13 coil (Supplementary Fig.~\ref{fig:pdb-rama-baseline}). This serves as the baseline for the conditioned structures.

Using $y^* = 592$ as the entropy budget, we solved the constrained SDP at five conditioning levels $\varepsilon \in \{0.1, 0.3, 0.5, 0.7, 0.9\}$. At each $\varepsilon$, we: (i)~solve the unconstrained SDP at $y^*$; (ii)~solve the constrained SDP to obtain the conditioned spin configuration; (iii)~construct the sign-flipped Hamiltonian; and (iv)~recover backbone dihedral angles $(\phi, \psi)$ via Karplus inversion with blending weight $\alpha = 0.9$ (see Methods).

At $\varepsilon = 0.1$, the coarse secondary structure closely matches the PDB ground truth: 42 $\beta$-sheet (57\%), 25 $\alpha$-helix (34\%), 3 left-$\alpha$ (4\%), and 4 coil (5\%). At $\varepsilon = 0.9$, the classification is nearly identical: 42 $\beta$-sheet, 24 $\alpha$-helix, 5 left-$\alpha$, and 3 coil (Fig.~\ref{fig:eps-sweep-rama}). The preservation of Ubiquitin's $\beta$-sheet character across all conditioning levels is consistent with the known mixed-$\alpha$/$\beta$ fold and indicates that the SDP produces physically plausible secondary structure assignments.

\begin{figure}[htbp]
\centering
\includegraphics[width=0.85\textwidth]{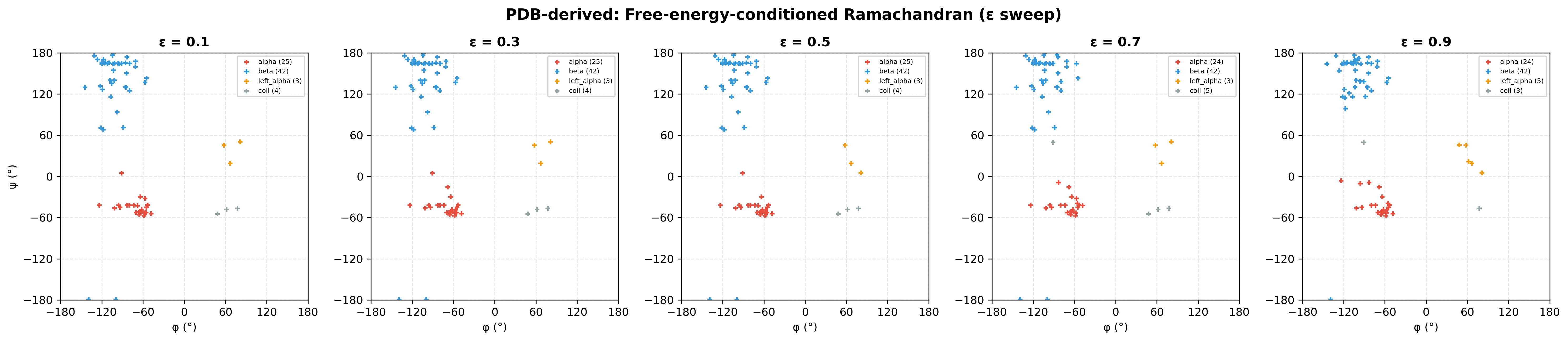}
\caption{Free-energy-conditioned Ramachandran plots~\cite{ramachandran1963} across the $\varepsilon$ sweep. Each panel shows the backbone dihedral angles $(\phi, \psi)$ for 74 resolved residues at the indicated conditioning level.}
\label{fig:eps-sweep-rama}
\end{figure}

\subsection{Structural Stability and Per-Residue Analysis}\label{sec:stability}

\paragraph{Helix-subtype stability.}
Decomposing the $\alpha$-helical residues into canonical subtypes, $\alpha_R$ (standard), $3_{10}$-helix, $\pi$-helix, and $\alpha$-broad, following the geometry of Pauling, Corey, and Branson~\cite{pauling1951}, reveals further resolution (Fig.~\ref{fig:ss-vs-eps}, Table~\ref{tab:subtypes}).

\begin{table}[htbp]
\centering
\caption{Helix subtype composition as a function of conditioning parameter $\varepsilon$. Only 1--2 residue shifts occur between consecutive $\varepsilon$ levels.}
\label{tab:subtypes}
\begin{tabular}{lccccc}
\toprule
\textbf{Subtype} & $\varepsilon{=}0.1$ & $\varepsilon{=}0.3$ & $\varepsilon{=}0.5$ & $\varepsilon{=}0.7$ & $\varepsilon{=}0.9$ \\
\midrule
$\alpha_R$ (standard) & 18 (24\%) & 17 (23\%) & 18 (24\%) & 17 (23\%) & 17 (23\%) \\
$3_{10}$-helix & 1 (1\%) & 2 (3\%) & 1 (1\%) & 2 (3\%) & 2 (3\%) \\
$\alpha$-broad & 6 (8\%) & 6 (8\%) & 6 (8\%) & 5 (7\%) & 5 (7\%) \\
$\pi$-helix & 0 & 0 & 0 & 0 & 0 \\
$\beta$-sheet & 42 (57\%) & 42 (57\%) & 42 (57\%) & 42 (57\%) & 42 (57\%) \\
Left-$\alpha$ & 3 (4\%) & 3 (4\%) & 3 (4\%) & 3 (4\%) & 5 (7\%) \\
Coil & 4 (5\%) & 4 (5\%) & 4 (5\%) & 5 (7\%) & 3 (4\%) \\
\bottomrule
\end{tabular}
\end{table}

The dominant $\alpha_R$ count varies by at most 1 residue (17--18) across the full $\varepsilon$ range; $\beta$-sheet is invariant at 42 residues; and the only changes involve 1--2 residues exchanging between $3_{10}$-helix, $\alpha$-broad, left-$\alpha$, and coil. This stability is consistent with the energy-entropy competition in equation~\eqref{Psi:eqn}: the ground-state Hamiltonian constrains the conformational space even under loose conditioning ($\varepsilon = 0.1$).

\begin{figure}[htbp]
\centering
\includegraphics[width=0.8\textwidth]{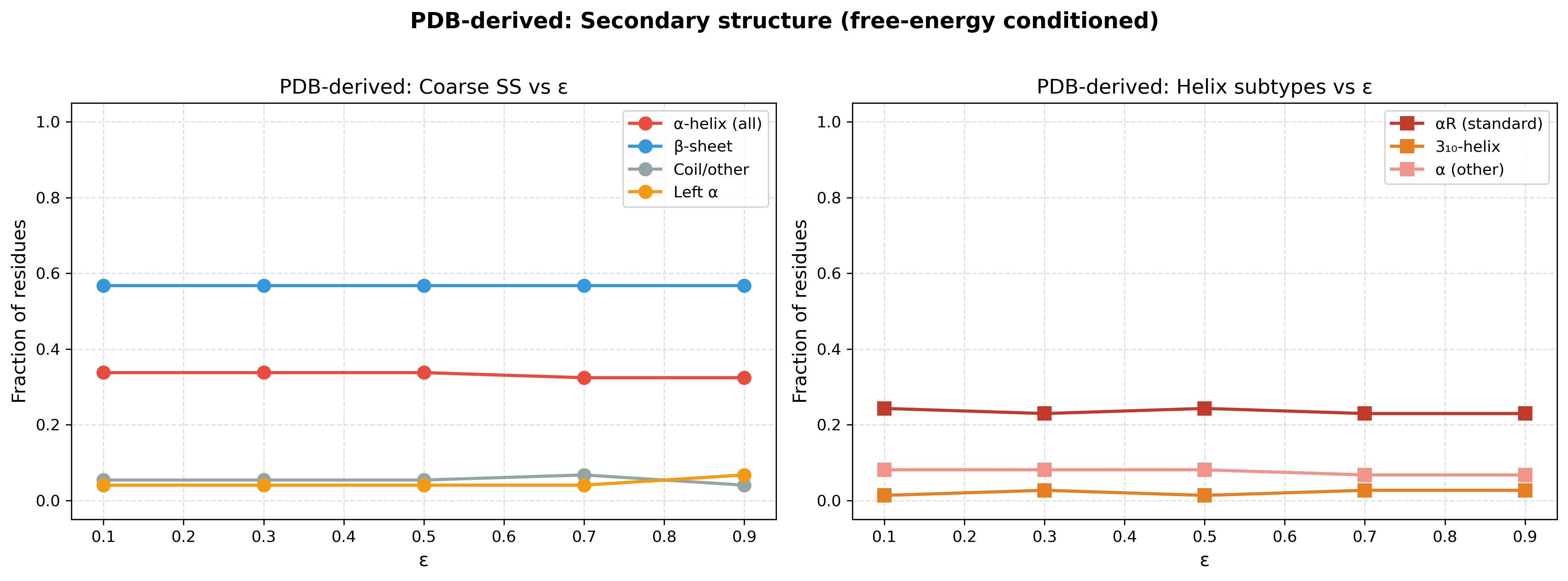}
\caption{Secondary structure composition vs.\ $\varepsilon$. \textbf{Left:} Coarse classification ($\beta$-sheet at 57\%, $\alpha$-helix at 34\%) is flat across conditioning levels. \textbf{Right:} Helix subtypes ($\alpha_R$, $3_{10}$, $\alpha$-broad) show only 1--2 residue fluctuations.}
\label{fig:ss-vs-eps}
\end{figure}

\paragraph{Per-residue strip diagram.}
Figure~\ref{fig:ss-strip} shows, for each residue, how its secondary structure assignment evolves with $\varepsilon$. The PDB ground truth appears as the bottom row. Subtype transitions at $\varepsilon \ge 0.7$ are localized to specific residues rather than distributed uniformly.

\begin{figure}[htbp]
\centering
\includegraphics[width=0.8\textwidth]{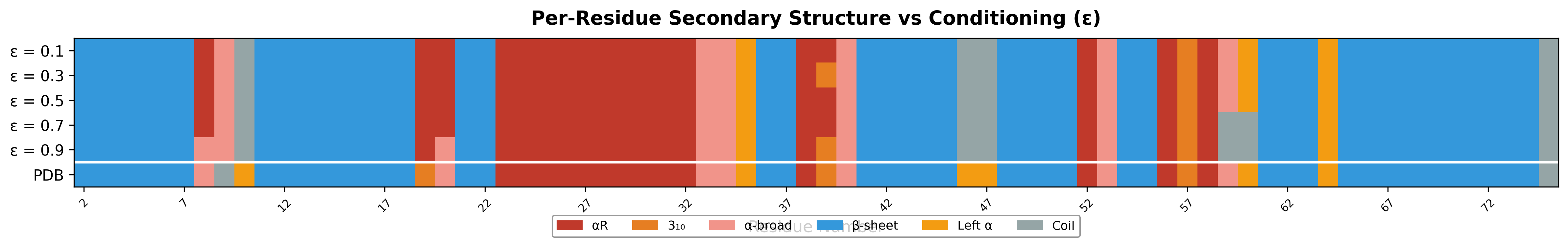}
\caption{Per-residue secondary structure strip diagram across the $\varepsilon$ sweep. Each column is one residue; colour encodes secondary structure subtype. The PDB ground truth is the bottom row. Transitions at $\varepsilon \ge 0.7$ are localized to specific residues.}
\label{fig:ss-strip}
\end{figure}

\paragraph{Per-residue angular drift.}
We computed the per-residue $\psi$ drift $\Delta\psi_i(\varepsilon)$ relative to $\varepsilon = 0.1$ (Fig.~\ref{fig:drift-heatmap}). The $\phi$ drift is identically zero at all conditioning levels because the ${}^3J$ Karplus constraint fully determines $\phi$ from the local coupling; it is therefore omitted.

The $\psi$ drift concentrates on approximately 20 residues at $\varepsilon \ge 0.7$, near positions 1, 6--7, 14--16, 40, 47--50, 62--63, and 67--75, with shifts up to $30^\circ$. These correspond to the residues undergoing the $\alpha_R \to \alpha$-broad / $3_{10}$ transitions in Table~\ref{tab:subtypes}. Residues in the main $\alpha$-helix (positions 23--34) show minimal drift, consistent with the known stability of this element in the Ubiquitin fold~\cite{vijaykumar1987}. The localization of angular drift to specific residues, rather than uniform perturbation of the backbone, is characteristic of physically meaningful structural transitions.

\begin{figure}[htbp]
\centering
\includegraphics[width=0.8\textwidth]{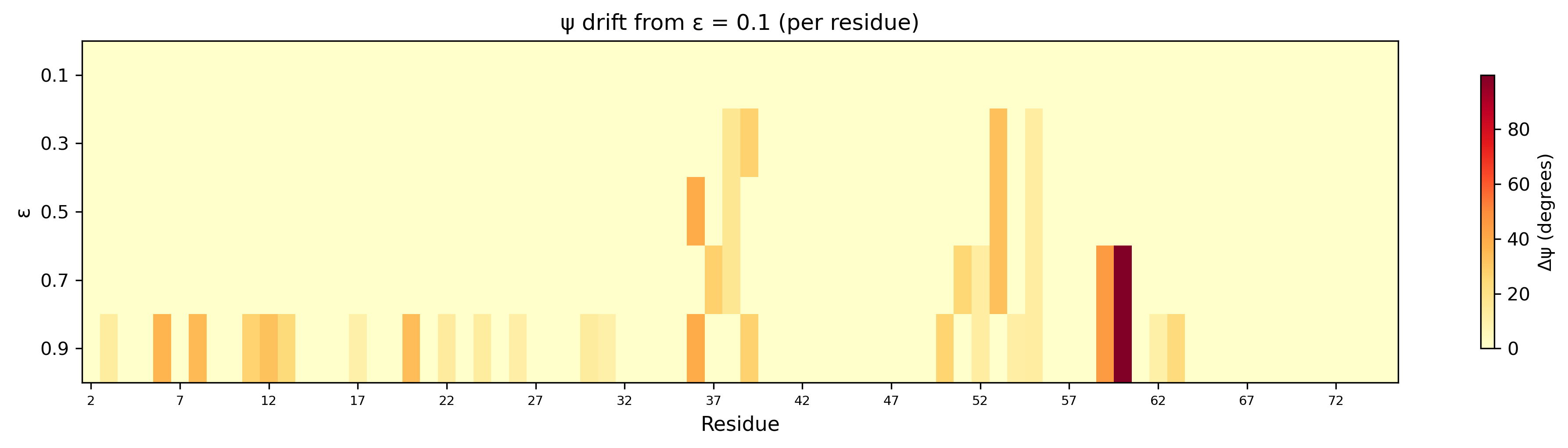}
\caption{Per-residue $\psi$ drift relative to $\varepsilon = 0.1$. Rows correspond to conditioning levels; columns to residues. Approximately 20 residues shift by up to $30^\circ$ at $\varepsilon \ge 0.7$. The $\phi$ drift (not shown) is identically zero.}
\label{fig:drift-heatmap}
\end{figure}

\subsection{Validation}\label{sec:validation}

\paragraph{Pipeline self-consistency.}
At $\varepsilon = 0.9$ (nearest to the ground state), all 75 $\phi$ and 75 $\psi$ angles are recovered within $0.05^\circ$ of the crystallographic values (Supplementary Fig.~\ref{fig:roundtrip}), confirming that the pipeline introduces negligible numerical error.

\paragraph{Root Mean Square Deviation (RMSD) against the crystal structure.}
The circular RMSD between SDP-recovered $\psi$ angles and the 1UBQ crystal structure is approximately $28^\circ$ for $\varepsilon \in \{0.1, 0.3, 0.5, 0.7\}$ and decreases to $20.6^\circ$ at $\varepsilon = 0.9$ (Supplementary Fig.~\ref{fig:rmsd-vs-eps}). The $\phi$ RMSD is identically zero across all $\varepsilon$, consistent with the ${}^3J$ Karplus locking observed in the drift analysis.

\paragraph{Energy-only vs.\ free-energy comparison.}
A side-by-side comparison at $\varepsilon = 0.9$ using $y = n = 76$ (energy-only) versus $y = y^* \approx 592$ (free-energy optimum) shows that the free-energy-conditioned structure produces a more physically realistic Ramachandran distribution, reflecting the entropy-energy balance at finite temperature (Fig.~\ref{fig:energy-vs-fe}).

\begin{figure}[htbp]
\centering
\includegraphics[width=0.8\textwidth]{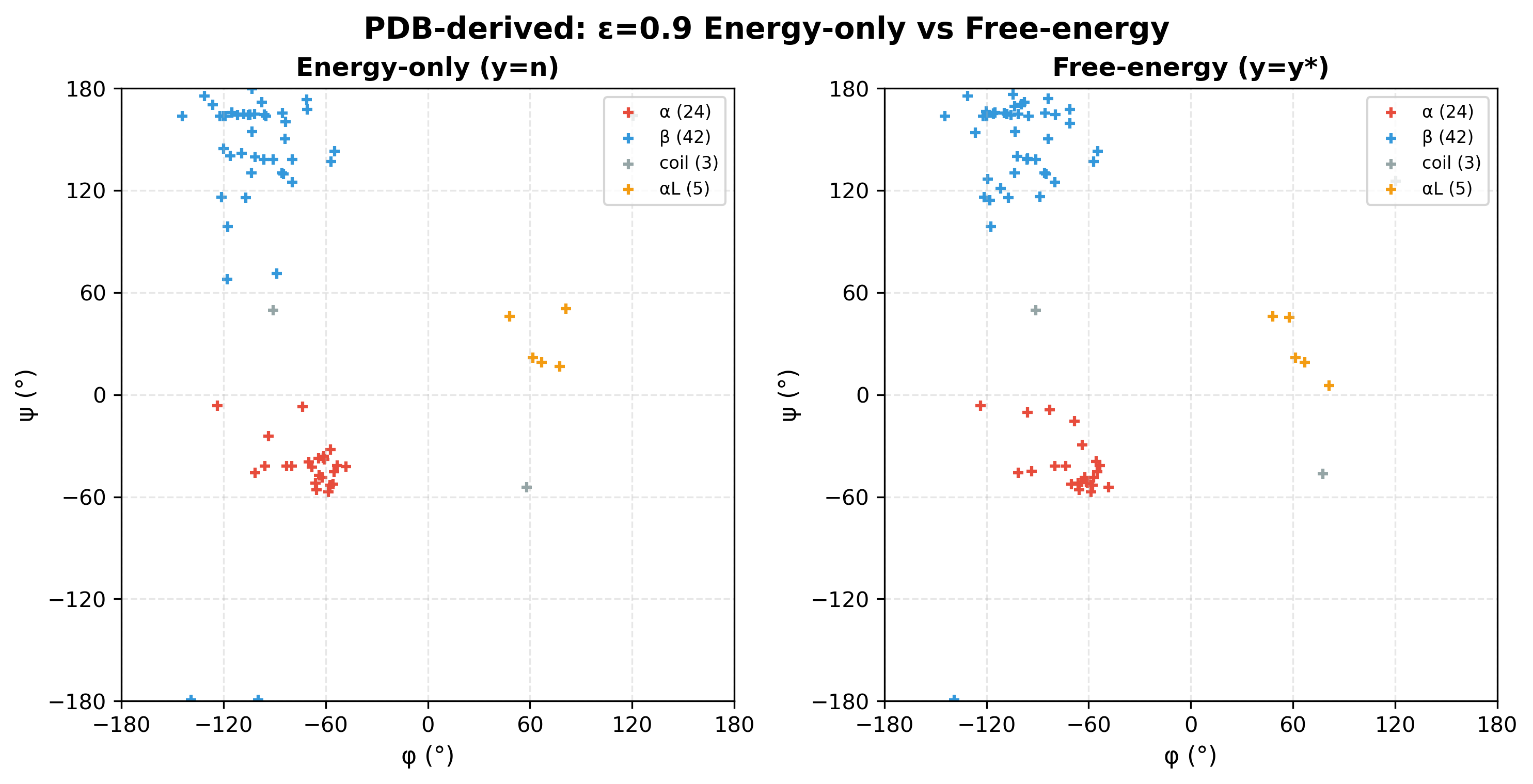}
\caption{Ramachandran plots at $\varepsilon = 0.9$: energy-only ($y = n$, left) vs.\ free-energy ($y = y^*$, right).}
\label{fig:energy-vs-fe}
\end{figure}

\paragraph{Unbalance factor on the protein Hamiltonian.}
Figure~\ref{fig:pdb-bias} shows the relative unbalance factor as a function of $\varepsilon$ for the Ubiquitin Hamiltonian. The factor decreases smoothly from $\sim$1.8 at $\varepsilon = 0.1$ to $\sim$1.2 at $\varepsilon = 0.9$, consistent with the synthetic $\ell_1$-PCA validation (Supplementary Fig.~\ref{fig:l1pca_quad}d).

\begin{figure}[htbp]
\centering
\includegraphics[width=0.45\textwidth]{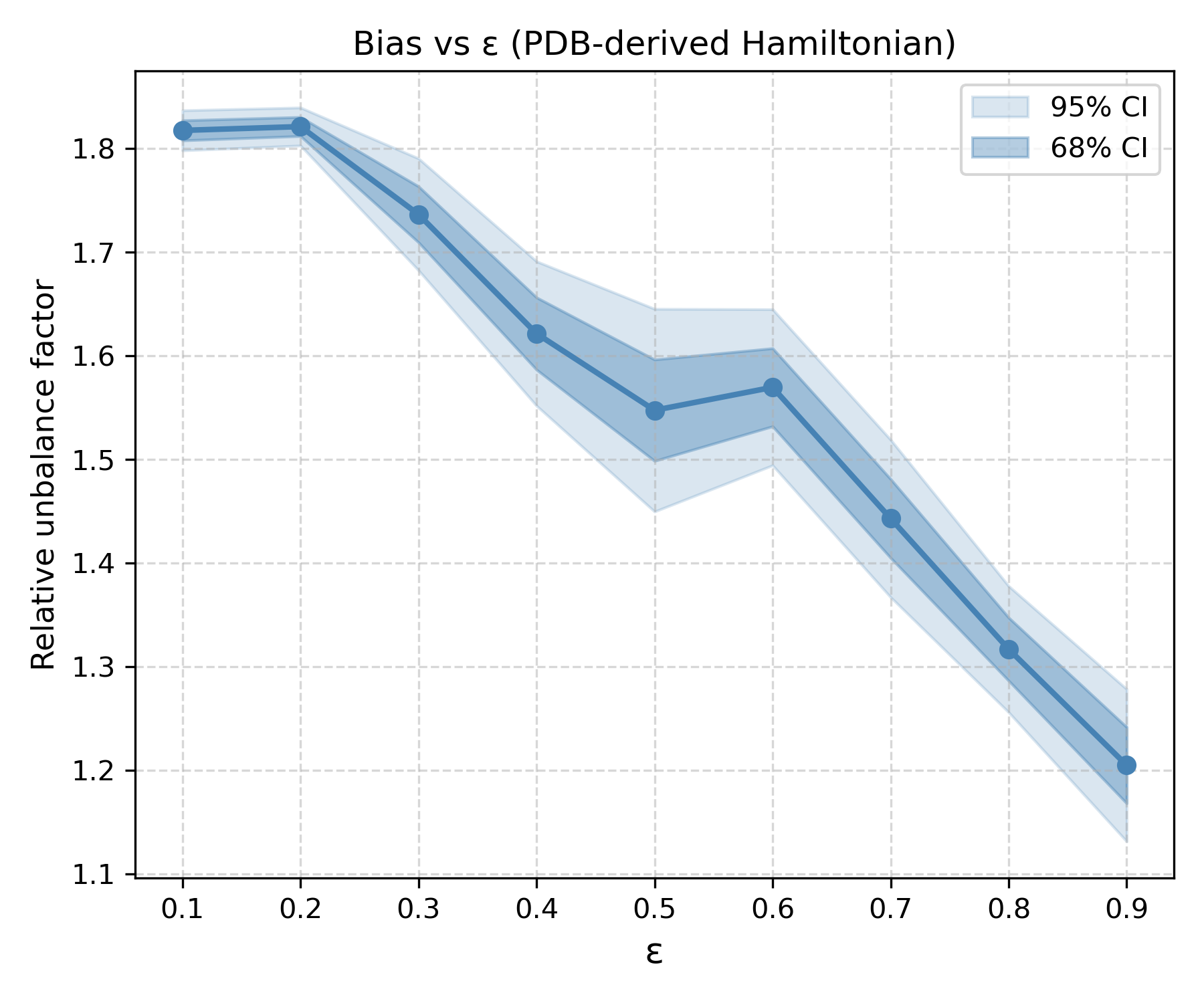}
\caption{Relative unbalance factor vs.\ $\varepsilon$ for the Ubiquitin Hamiltonian. Shaded regions show 68\% and 95\% confidence intervals over Gaussian rounding draws.}
\label{fig:pdb-bias}
\end{figure}

\section{Discussion}\label{sec:discussion}

The results above demonstrate that an SDP framework can be used to compute conditioned free-energy states of protein spin Hamiltonians with provable guarantees. Three aspects merit further discussion: the physical interpretation of the conditioning, the structural selectivity of the results, and the limitations of the current approach.

\paragraph{Physical interpretation of the $\varepsilon$-sweep.}
The parameter $\varepsilon$ controls how close the conditioned solution must remain to the unconstrained optimum via the constraint $\|Ax\|_1 \ge \varepsilon \cdot \OPT$. Large $\varepsilon$ yields structures closest to the native fold; small $\varepsilon$ permits exploration of unbalanced configurations. The stability of secondary structure across the sweep, with $\beta$-sheet completely invariant and helix subtypes shifting by at most 1--2 residues, indicates strong structural constraints imposed by the ground-state Hamiltonian. This stability extends to variation in the blending weight $\alpha$: coarse-level agreement between $\alpha = 0.9$ and $\alpha \in \{0.7, 0.8\}$ exceeds 95\% at every $\varepsilon$ (Supplementary Section~\ref{sec:alpha-sensitivity}).

\paragraph{Selective $\psi$ sensitivity.}
The $\phi$/$\psi$ asymmetry observed in Section~\ref{sec:stability} has a structural interpretation. The $\phi$ angles are locked because each $\phi$ is determined by a single intra-residue ${}^3J$ coupling that is never modified during conditioning. In contrast, $\psi$ angles depend on sequential ${}^2J$ couplings that are subject to sign-flipping, making them responsive to changes in the Hamiltonian. The concentration of drift on $\sim$20 specific residues, rather than uniform perturbation, suggests that these residues occupy shallow regions of the local free-energy surface and may correspond to functionally relevant conformational flexibility.

\paragraph{Validation considerations.}
The conditioned structures represent predictions of higher-energy conformational states that are difficult to validate independently. X-ray crystallography captures only the ground-state minimum, and NMR relaxation dispersion (CPMG, CEST) yields only sparse constraints~\cite{palmer2004}. Direct comparison with molecular dynamics (MD)-sampled excited states or experimentally resolved transition-state structures remains an important future direction.

\paragraph{Limitations.}
On the theoretical side, the guarantee for Algorithm~\ref{algmain} i.e., the "approximation factor" depends on the optimal instance, which is unknown. Ideally, one would have an explicit constant factor approximation guarantee. This is a particular hard problem, as the guarantee has to be strictly better than a 2-approximation for it to be useful. The AM-GM ratio can be close to 1, hence it is still meaningful. For the application to conditioned protein structure prediction, the coarse-grained spin model ($\pm 1$ per residue) simplifies the continuous dihedral degrees of freedom and many-body interactions of real proteins. Furthermore, the ${}^2J$ Karplus coefficients were fitted to the Ubiquitin crystal structure; generalization requires re-fitting or universal parameters. Moreover, the non-injective Karplus equation is resolved by nearest-branch selection relative to known PDB structure, limiting the current framework to small perturbative exploration around a known reference. Finally, the conditioned structures have not been independently validated against experimental results.

\paragraph{Future directions.}
Integration with diffusion-based structure prediction~\cite{baker} could combine the SDP's landscape exploration with 3D coordinate generation. The mathematical framework applies to any quadratic spin Hamiltonian, suggesting applications beyond protein structure to frustrated spin systems in materials science.

\section*{Methods}

\subsection*{Karplus equations and inversion}

Backbone dihedral angles are recovered from J-coupling values by inverting two Karplus-type equations: a three-bond (${}^3J$) equation for $\phi$~\cite{karplus1959,vuister1993} and a two-bond (${}^2J$) equation for $\psi$~\cite{wirmer2002}.

The $\phi$ angle is determined from three-bond couplings:
\begin{equation}
{}^3J(\phi) = 6.51\cos^2(\phi - 60^\circ) - 1.76\cos(\phi - 60^\circ) - 1.60.
\end{equation}

The $\psi$ angle is determined from two-bond ${}^2J_{\text{C}\alpha\text{-N}}$ couplings~\cite{wirmer2002}:
\begin{equation}
{}^2J_{\text{C}\alpha\text{-N}}(\psi) = A\cos^2(\psi + \delta) + B\cos(\psi + \delta) + C,
\end{equation}
with fitted values $A\approx 4.95$, $B\approx 3.64$, $C\approx -7.95$, and $\delta\approx 119.1^\circ$. These coefficients were fitted by least-squares regression using all 75 $\psi$ angles from the 1UBQ crystal structure.

For conditioned structures, a convex blending with weight $\alpha = 0.9$ is applied at sign-flipped residue pairs:
\begin{equation}
J_{\text{eff}} = \alpha\,J + (1-\alpha)\,(-J) = (2\alpha - 1)\,J,
\end{equation}
so that $J_{\text{eff}} = 0.8\,J$ for flipped pairs and $J_{\text{eff}} = J$ otherwise. The ${}^3J$ couplings are never modified.

\subsection*{Secondary structure classification}

Coarse classification follows standard Ramachandran region boundaries~\cite{ramachandran1963,lovell2003}. Fine-grained classification decomposes the $\alpha$-helical region into subtypes based on the canonical helix geometries~\cite{pauling1951}:
\begin{itemize}
  \item $\alpha_R$: $\phi \in [-80^\circ, -40^\circ]$, $\psi \in [-60^\circ, -30^\circ]$
  \item $3_{10}$: $\phi \in [-80^\circ, -40^\circ]$, $\psi \in [-30^\circ, -10^\circ]$
  \item $\pi$: $\phi \in [-80^\circ, -40^\circ]$, $\psi \in [-80^\circ, -60^\circ]$
  \item $\alpha$-broad: remaining residues with $\phi \in [-160^\circ, -20^\circ]$, $\psi \in [-80^\circ, 10^\circ]$
  \item $\beta$-sheet: $\phi \in [-180^\circ, -40^\circ]$, $\psi \in [60^\circ, 180^\circ] \cup [-180^\circ, -120^\circ]$
  \item Left-$\alpha$: $\phi \in [20^\circ, 100^\circ]$, $\psi \in [0^\circ, 80^\circ]$
  \item Coil: all remaining residues
\end{itemize}

\subsection*{Free-energy sweep}

The entropy budget $y$ is swept from 1 to $20n$ (with adaptive $3\times$ extension if $y^*$ hits the boundary) and $F(y) = \max\{\|Az\|_1 : \|z\|_2^2 \le y\} - y$ is computed at each grid point.




\subsection*{Author contributions}
P.W. conceived the study, developed the theoretical framework, designed and implemented the computational pipeline, performed the experiments, and wrote the manuscript. S.K. helped design the first version of the SDP algorithm and S.V. helped with the proof of connection between $2\to 1$ norm computation and free-energy computation.


\clearpage
\bibliographystyle{plain}
\bibliography{sample}

@article{vijaykumar1987,
    author  = {Vijay-Kumar, S. and Bugg, Charles E. and Cook, William J.},
    title   = {Structure of ubiquitin refined at 1.8 {\AA} resolution},
    year    = {1987},
    journal = {Journal of Molecular Biology},
    volume  = {194},
    number  = {3},
    pages   = {531--544},
    doi     = {10.1016/0022-2836(87)90679-6}
}

@article{baker,
author={Joseph L. Watson and David Juergens and Nathaniel R. Bennett and Brian L. Trippe and Jason Yim and Helen E. Eisenach and Woody Ahern and Andrew J. Borst and Robert J. Ragotte and Lukas F. Milles and Basile I. M. Wicky and Nikita Hanikel and Samuel J. Pellock and Alexis Courbet and William Sheffler and Jue Wang and Preetham Venkatesh and Isaac Sappington and Susana Vázquez Torres and Anna Lauko and Valentin De Bortoli and Emile Mathieu and Sergey Ovchinnikov and Regina Barzilay and Tommi S. Jaakkola and Frank DiMaio and Minkyung Baek and David Baker},
title={{De novo design of protein structure and function with RFdiffusion}},
journal={Nature},
year={2023},
volume={620}
}

@article{alon-naor,
author = {Alon, Noga and Naor, Assaf},
title = {Approximating the Cut-Norm via Grothendieck's Inequality},
journal = {SIAM Journal on Computing},
volume = {35},
number = {4},
pages = {787-803},
year = {2006}
}

@article{cw1,
title={{Ising Critical Behavior of Inhomogeneous Curie-Weiss Models and Annealed Random Graphs}},
author={Sander Dommers and Cristian Giardina and Claudio Giberti and Remco van der Hofstad and Maria Luisa Prioriello},
journal={Communications in Mathematical Physics},
year={2016},
volume={348}
}

@article{cw2,
title={{Annealed central limit theorems for the Ising
model on random graphs}},
author={Cristian Giardina and Claudio Giberti and Remco van der Hofstad and Maria Luisa Prioriello},
journal={Latin American Journal of Probability and Mathematical Statistics},
year={2016},
volume={XIII}
}

@article{gw,
author = {Goemans, Michel X. and Williamson, David P.},
title = {{Improved approximation algorithms for maximum cut and satisfiability problems using semidefinite programming}},
year = {1995},
issue_date = {Nov. 1995},
publisher = {Association for Computing Machinery},
address = {New York, NY, USA},
volume = {42},
number = {6},
issn = {0004-5411},
url = {https://doi.org/10.1145/227683.227684},
doi = {10.1145/227683.227684},
journal = {J. ACM},
month = nov,
pages = {1115–1145},
numpages = {31}
}

@article{bhat,
author = {Bhattiprolu, Vijay and Ghosh, Mrinal Kanti and Guruswami, Venkatesan and Lee, Euiwoong and Tulsiani, Madhur},
title = {Inapproximability of Matrix \(\boldsymbol{p \rightarrow q}\) Norms},
journal = {SIAM Journal on Computing},
volume = {52},
number = {1},
pages = {132-155},
year = {2023}
}

@ARTICLE{l1pca1,
  author={Markopoulos, Panos P. and Karystinos, George N. and Pados, Dimitris A.},
  journal={IEEE Transactions on Signal Processing}, 
  title={{Optimal Algorithms for  $L_{1}$-subspace Signal Processing}}, 
  year={2014},
  volume={62},
  number={19},
  pages={5046-5058}
}

@article{khot-naor,
author = {Khot, Subhash and Naor, Assaf},
title = {Grothendieck-Type Inequalities in Combinatorial Optimization},
journal = {Communications on Pure and Applied Mathematics},
volume = {65},
number = {7},
pages = {992-1035},
year = {2012}
}

@inproceedings{cook,
author = {Cook, Stephen A.},
title = {The complexity of theorem-proving procedures},
year = {1971},
isbn = {9781450374644},
publisher = {Association for Computing Machinery},
address = {New York, NY, USA},
booktitle = {Proceedings of the Third Annual ACM Symposium on Theory of Computing},
pages = {151–158},
numpages = {8},
location = {Shaker Heights, Ohio, USA},
series = {STOC '71}
}

@article{hartnew,
author={William Hart and Alantha Newman},
title={{Protein Structure Prediction with Lattice Models}},
journal={Handbook of Computational Molecular Biology},
publisher={Chapman and Hall/CRC},
year={2005}
}

@article{bhatarx,
  author       = {Vijay Bhattiprolu and
                  Mrinalkanti Ghosh and
                  Venkatesan Guruswami and
                  Euiwoong Lee and
                  Madhur Tulsiani},
  title        = {Inapproximability of Matrix p{\(\rightarrow\)}q Norms},
  journal      = {CoRR},
  volume       = {abs/1802.07425},
  year         = {2018},
  url          = {http://arxiv.org/abs/1802.07425},
  eprinttype    = {arXiv},
  eprint       = {1802.07425}
}

@article{bhatarx2,
  author       = {Vijay Bhattiprolu and
                  Mrinalkanti Ghosh and
                  Venkatesan Guruswami and
                  Euiwoong Lee and
                  Madhur Tulsiani},
  title        = {Approximating Operator Norms via Generalized Krivine Rounding},
  journal      = {CoRR},
  volume       = {abs/1804.03644},
  year         = {2018},
  url          = {http://arxiv.org/abs/1804.03644},
  eprinttype    = {arXiv},
  eprint       = {1804.03644}
}

@article{karplus1959,
  author    = {Martin Karplus},
  title     = {{Contact Electron-Spin Coupling of Nuclear Magnetic Moments}},
  journal   = {The Journal of Chemical Physics},
  volume    = {30},
  number    = {1},
  pages     = {11--15},
  year      = {1959},
  doi       = {10.1063/1.1730075}
}

@article{vuister1993,
  author    = {Geerten W. Vuister and Ad Bax},
  title     = {{Quantitative J correlation: a new approach for measuring homonuclear three-bond $J(\mathrm{HN}H\alpha)$ coupling constants in $^{15}$N-enriched proteins}},
  journal   = {Journal of the American Chemical Society},
  volume    = {115},
  number    = {17},
  pages     = {7772--7777},
  year      = {1993},
  doi       = {10.1021/ja00070a024}
}

@article{wirmer2002,
  author    = {Julia Wirmer and Harald Schwalbe},
  title     = {{Angular dependence of ${}^1J(N_i,C_{\alpha i})$ and ${}^2J(N_i,C_{\alpha(i-1)})$ coupling constants measured in J-modulated HSQCs}},
  journal   = {Journal of Biomolecular NMR},
  volume    = {23},
  pages     = {47--55},
  year      = {2002},
  doi       = {10.1023/A:1015384805098}
}

@article{ramachandran1963,
  author    = {G. N. Ramachandran and C. Ramakrishnan and V. Sasisekharan},
  title     = {{Stereochemistry of polypeptide chain configurations}},
  journal   = {Journal of Molecular Biology},
  volume    = {7},
  number    = {1},
  pages     = {95--99},
  year      = {1963},
  doi       = {10.1016/S0022-2836(63)80023-6}
}

@article{konnoyamazaki1991,
  author    = {Hiroshi Konno and Hiroaki Yamazaki},
  title     = {Mean-absolute deviation portfolio optimization model and its applications to {T}okyo stock market},
  journal   = {Management Science},
  volume    = {37},
  number    = {5},
  pages     = {519--531},
  year      = {1991},
  doi       = {10.1287/mnsc.37.5.519}
}

@article{ustunrudin2016,
  author    = {Berk Ustun and Cynthia Rudin},
  title     = {Supersparse linear integer models for optimized medical scoring systems},
  journal   = {Machine Learning},
  volume    = {102},
  number    = {3},
  pages     = {349--391},
  year      = {2016},
  doi       = {10.1007/s10994-015-5528-6}
}

@article{meinshausen2013,
  author    = {Nicolai Meinshausen},
  title     = {Sign-constrained least squares estimation for high-dimensional regression},
  journal   = {Electronic Journal of Statistics},
  pages     = {1607--1631},
  year      = {2013}
}

@article{nnls,
    author = {Martin Slawski and Matthias Hein},
    title = {{Non-negative least squares for high-dimensional linear models: Consistency and sparse recovery without regularization}},
    journal = {Electronic Journal of Statistics},
    year = {2013}
}

@article{alphafold2,
  author    = {John Jumper and Richard Evans and Alexander Pritzel and Tim Green and Michael Figurnov and Olaf Ronneberger and Kathryn Tunyasuvunakool and Russ Bates and Augustin {\v{Z}}{\'i}dek and Anna Potapenko and Alex Bridgland and Clemens Meyer and Simon A. A. Kohl and Andrew J. Ballard and Andrew Cowie and Bernardino Romera-Paredes and Stanislav Nikolov and Rishub Jain and Jonas Adler and Trevor Back and Stig Petersen and David Reiman and Ellen Clancy and Michal Zielinski and Martin Steinegger and Michalina Pacholska and Tamas Berghammer and Sebastian Bodenstein and David Silver and Oriol Vinyals and Andrew W. Senior and Koray Kavukcuoglu and Pushmeet Kohli and Demis Hassabis},
  title     = {Highly accurate protein structure prediction with {AlphaFold}},
  journal   = {Nature},
  volume    = {596},
  number    = {7873},
  pages     = {583--589},
  year      = {2021},
  doi       = {10.1038/s41586-021-03819-2}
}

@article{opls4,
  author    = {Chuanjie Lu and Chao Wu and Dmitry Ghoreishi and Wei Chen and Lingle Wang and Wolfgang Damm and Gail A. Ross and Markus K. Dahlgren and Ellery Russell and Christopher D. Von Bargen and Robert Abel and Richard A. Friesner and Edward D. Harder},
  title     = {{OPLS4}: Improving Force Field Accuracy on Challenging Regimes of Chemical Space},
  journal   = {Journal of Chemical Theory and Computation},
  volume    = {17},
  number    = {7},
  pages     = {4291--4300},
  year      = {2021},
  doi       = {10.1021/acs.jctc.1c00302}
}

@article{barahona1982,
  author    = {Francisco Barahona},
  title     = {On the computational complexity of {I}sing spin glass models},
  journal   = {Journal of Physics A: Mathematical and General},
  volume    = {15},
  number    = {10},
  pages     = {3241--3253},
  year      = {1982},
  doi       = {10.1088/0305-4470/15/10/028}
}

@article{jerrum1993,
  author    = {Mark Jerrum and Alistair Sinclair},
  title     = {Polynomial-time approximation algorithms for the {I}sing model},
  journal   = {SIAM Journal on Computing},
  volume    = {22},
  number    = {5},
  pages     = {1087--1116},
  year      = {1993},
  doi       = {10.1137/0222066}
}

@article{talos,
  author    = {Yang Shen and Ad Bax},
  title     = {{TALOS+}: a hybrid method for predicting protein backbone torsion angles from {NMR} chemical shifts},
  journal   = {Journal of Biomolecular NMR},
  volume    = {44},
  number    = {4},
  pages     = {213--223},
  year      = {2009},
  doi       = {10.1007/s10858-009-9333-z}
}

@article{csrosetta,
  author    = {Yang Shen and Oliver Lange and Frank Delaglio and Paolo Rossi and James M. Aramini and Gaohua Liu and Alexander Eletsky and Yibing Wu and Kiran K. Singarapu and Alexander Lemak and Alexandr Ignatchenko and Cheryl H. Arrowsmith and Thomas Szyperski and Gaetano T. Montelione and David Baker and Ad Bax},
  title     = {Consistent blind protein structure generation from {NMR} chemical shift data},
  journal   = {Proceedings of the National Academy of Sciences},
  volume    = {105},
  number    = {12},
  pages     = {4685--4690},
  year      = {2008},
  doi       = {10.1073/pnas.0800256105}
}

@book{pathria2011,
  author    = {R. K. Pathria and Paul D. Beale},
  title     = {Statistical Mechanics},
  publisher = {Academic Press},
  edition   = {3rd},
  year      = {2011}
}

@article{pauling1951,
  author    = {Linus Pauling and Robert B. Corey and H. R. Branson},
  title     = {The structure of proteins: two hydrogen-bonded helical configurations of the polypeptide chain},
  journal   = {Proceedings of the National Academy of Sciences},
  volume    = {37},
  number    = {4},
  pages     = {205--211},
  year      = {1951},
  doi       = {10.1073/pnas.37.4.205}
}

@article{lovell2003,
  author    = {Simon C. Lovell and Ian W. Davis and W. Bryan Arendall and Paul I. W. de Bakker and J. Michael Word and Michael G. Prisant and Jane S. Richardson and David C. Richardson},
  title     = {Structure validation by {C}$\alpha$ geometry: $\phi$, $\psi$ and {C}$\beta$ deviation},
  journal   = {Proteins: Structure, Function, and Bioinformatics},
  volume    = {50},
  number    = {3},
  pages     = {437--450},
  year      = {2003},
  doi       = {10.1002/prot.10286}
}

@article{rosetta,
  author    = {Andrew Leaver-Fay and Michael Tyka and Steven M. Lewis and Oliver F. Lange and James Thompson and Ron Jacak and Kristian W. Kaufman and P. Douglas Renfrew and Colin A. Smith and Will Sheffler and Ian W. Davis and Seth Cooper and Adrien Treuille and Daniel J. Mandell and Florian Richter and Yih-En Andrew Ban and Sarel J. Fleishman and Jacob E. Corn and David E. Kim and Sergey Lyskov and Monica Berrondo and Stuart Mentzer and Zoran Popovi{\'c} and James J. Havranek and John Karanicolas and Rhiju Das and Jens Meiler and Tanja Kortemme and Jeffrey J. Gray and Brian Kuhlman and David Baker and Philip Bradley},
  title     = {{ROSETTA3}: an object-oriented software suite for the simulation and design of macromolecules},
  journal   = {Methods in Enzymology},
  volume    = {487},
  pages     = {545--574},
  year      = {2011},
  doi       = {10.1016/B978-0-12-381270-4.00019-6}
}

@article{palmer2004,
  author    = {Arthur G. Palmer and Charles D. Kroenke and J. Patrick Loria},
  title     = {Nuclear magnetic resonance methods for quantifying microsecond-to-millisecond motions in biological macromolecules},
  journal   = {Methods in Enzymology},
  volume    = {339},
  pages     = {204--238},
  year      = {2001},
  doi       = {10.1016/S0076-6879(01)39315-1}
}

@article{rfold,
author = {Minkyung Baek  and Frank DiMaio  and Ivan Anishchenko  and Justas Dauparas  and Sergey Ovchinnikov  and Gyu Rie Lee  and Jue Wang  and Qian Cong  and Lisa N. Kinch  and R. Dustin Schaeffer  and Claudia Millán  and Hahnbeom Park  and Carson Adams  and Caleb R. Glassman  and Andy DeGiovanni  and Jose H. Pereira  and Andria V. Rodrigues  and Alberdina A. van Dijk  and Ana C. Ebrecht  and Diederik J. Opperman  and Theo Sagmeister  and Christoph Buhlheller  and Tea Pavkov-Keller  and Manoj K. Rathinaswamy  and Udit Dalwadi  and Calvin K. Yip  and John E. Burke  and K. Christopher Garcia  and Nick V. Grishin  and Paul D. Adams  and Randy J. Read  and David Baker },
title = {Accurate prediction of protein structures and interactions using a three-track neural network},
journal = {Science},
volume = {373},
number = {6557},
pages = {871-876},
year = {2021}
}

\clearpage
\appendix
\section{Supplementary Information}

\subsection{Notation}\label{sec:notation}

\begin{table}[htbp]
\centering
\caption{Summary of notation used in the paper.}
\label{tab:notation}
\begin{tabular}{ll}
\toprule
\textbf{Symbol} & \textbf{Description} \\
\midrule
$A \in \RR^{m \times m}$ & Square root of positive definite coupling matrix \\
$H(x)$ & Spin Hamiltonian (Equation~\ref{eqn:ham}) \\
$x \in \{\pm 1\}^n$ & Spin configuration \\
$Z$ & Partition function \\
$\Psi(m)$ & Log-cumulant: $\lim_{n\to\infty} \frac{1}{n}\log Z$ \\
$F = U - TS$ & Helmholtz free-energy \\
$\beta = 1/T$ & Inverse temperature \\
$y$ & Entropy budget (trace constraint on $V$) \\
$y^*$ & Optimal entropy budget maximizing $F(y)$ \\
$\varepsilon$ & Conditioning parameter ($\varepsilon \in (0,1]$) \\
$\OPT$ & Unconstrained SDP optimum ($\ell_1$ objective) \\
$\SDP$ & Unbalance factor from Algorithm~\ref{algmain} \\
$\B$ & Bias (unbalance measure, Equation~\ref{unbal:eqn}) \\
$\delta$ & Unbalance with respect to matrix $B$ \\
$U, V \in \RR^{n \times n}$ & SDP solution matrices (positive semidefinite) \\
$\phi_i, \psi_i$ & Backbone dihedral angles for residue $i$ \\
${}^3J, {}^2J$ & Three-bond and two-bond J-coupling constants \\
$\alpha$ & Blending weight for conditioned Karplus inversion \\
$\step(\cdot)$ & Sign function \\
\bottomrule
\end{tabular}
\end{table}
\clearpage

\subsection{Theorem and Lemma Statements for Algorithm~\ref{algmain}}\label{sbs:thm-lm-stmt}
Recall that given a matrix $B\in\RR^{n\times n}$ and vector $x\in\RR^n$, let $B_i$ denote the $i^{th}$ row of $B$. We say that $x$ is {\em $\delta$-unbalanced with respect to $B$}, if:
\begin{equation}
    \delta=\frac{1}{n^2}\cdot\sum_{i,j\in[n]}\left(\frac{\step(\langle B_i,x\rangle)+\step(\langle B_j,x\rangle)}{2}\right)^2,\label{unbal-rep:eqn}
\end{equation}
where $\step(\cdot)$ denotes the $\pm 1$ valued standard sign function, i.e., $\step(x)=+1$ if and only if $x\ge 0$.

\noindent {\em Unbalanced $2\to1$ norm computation:} Given matrices $A, B\in\RR^{n\times n}$,\footnote{$A$ can be a rectangular, but to keep the number of variables tracked reasonable in the proofs, we make it square here.} we want to compute an optimal $x\in\RR^n$ that maximizes $\|Ax\|_1$, while ensuring:
\begin{itemize}
    \item $\|x\|\le 1$, where $\|\cdot\|$ denotes $2$-norm, and
    \item $x$ is $\delta$-unbalanced with respect to $B$, for a given $\delta\in[\frac{1}{2},1]$.
\end{itemize}
Since the above problem is computationally intractable~\cite{bhat}, we focus on heuristic algorithms. We designed an SDP based algorithm (Algorithm~\ref{algmain}) that sweeps the parameter space to compute an approximate solution. Theorem~\ref{main:thm} and its underlying lemmas are the main contributions stated in this section. Their proofs are given in Section~\ref{sbs:det-proofs}.

\paragraph{Non-linear optimization}
A reasonable proxy for maximizing the LHS of Equation~\ref{unbal:eqn} (equivalently Equation~\ref{unbal-rep:eqn}) is to maximize its expected value that arises during Gaussian rounding, i.e., design a randomized algorithm like~\cite{gw} that maximizes the amount of {\em expected} unbalance, while simultaneously ensuring that the $\ell_1$ objective remains $\Omega(\OPT)$, where $\OPT$ is its optimal value. The objective of such a SDP relaxation can be written explicitly using Grothendieck's identity (Lemma~\ref{lm:grot2}). The conclusion being that we can compute unbalanced solutions, which are maximally unbalanced in expectation by solving the following maximization problem:
\begin{eqnarray}
    \max_{U,V\in\RR^{n\times n}} & \frac{1}{2n^2}\cdot\sum_{i,j\in[n]}\left(1+\frac{\langle B_i,B_j \rangle_V}{\|B_i\|_V\cdot\|B_j\|_V}\right),\label{unbal2:eqn}\\
    \mathrm{s.t.}&\ \sum_{i,j\in[n]} A_{ij}\langle v_i,u_j\rangle\ge\varepsilon\cdot\OPT,\quad \varepsilon\in(0,1) \label{sdp-constr:eqn}\\
    \quad &\frac{1}{n}\cdot\sum_{i\in[n]}\|v_i\|^2\le 1,\label{trconstr2:eqn}\\
    \quad &\forall i\in[n]:\ \|u_i\|=1,\label{absconstr2:eqn}\\
    \quad &\forall i,j\in[n]:\ U_{ij}=\langle u_i,u_j\rangle,\ V_{ij}=\langle v_i,v_j \rangle,\label{defmat:eqn}
\end{eqnarray}
where $U, V\in\RR^{n\times n}$ are positive semidefinite matrices, and for vectors $z, z'\in\RR^n$, we denote $\|z\|_V^2:=\langle zV, z \rangle$ and $\langle z,z' \rangle_V:=\langle zV,z' \rangle$. The solution vectors $v_i$ and $u_i$ map to real valued variables $x_i$ and $y_i$ using a~\cite{gw} like Gaussian rounding algorithm (see Algorithm~\ref{alground}).\footnote{Note that a solution of the optimization problem specified by maximizing the objective in Equation~\ref{unbal2:eqn}, subject to constraints in Equations~\ref{sdp-constr:eqn},~\ref{trconstr2:eqn},~\ref{absconstr2:eqn} and~\ref{defmat:eqn} can be translated (up to constant factor loss) to a solution for a version of the unbalanced $2\to 1$ norm optimization problem that consists of maximizing the objective specified by the LHS of Equation~\ref{sdp-constr:eqn} subject to constraints in Equations~\ref{unbal2:eqn},~\ref{trconstr2:eqn},~\ref{absconstr2:eqn} and~\ref{defmat:eqn}, by simply searching over $\varepsilon\in(0,1]$ in the former problem. 
}

However, Equation~\ref{unbal2:eqn} is not concave, and neither can it be made concave by easy manipulation. An easy heuristic to linearize the objective is to assume $\forall i,j\in[n]:\ \|B_i\|_{V^*}=\|B_j\|_{V^*}$ holds for the optimal solution. But, that would be too restrictive. For example, the optimal solution matrix $V^*$ could have $o(1)$ fraction of the total trace allocated to a few large unequal values, and the rest of the indices have $\|B_i\|_{V^*}=\|B_j\|_{V^*}$. That would be missed by the overtly simple heuristic above. In Algorithm~\ref{algmain}, we present a solution that is substantially more general than the strategy above. 

We can approximately solve the non-convex optimization problem specified by Equations~\ref{unbal2:eqn},~\ref{sdp-constr:eqn},~\ref{trconstr2:eqn},~\ref{absconstr2:eqn} and~\ref{defmat:eqn} using SDP rounding in Algorithm~\ref{algmain}, which is parametrized by $\varepsilon\in(0,1]$. Higher the $\varepsilon$, closer is the rounded $\ell_1$ objective value that can be obtained using Algorithm~\ref{algmain} to $\OPT$. Theorem~\ref{main:thm}, provides a trade-off between the $\ell_1$ objective as well as the unbalance factor for Algorithm~\ref{algmain}, and Theorem~\ref{thm:polytime} shows that Algorithm~\ref{algmain} requires only polynomial time under a mild assumption (positive definiteness of $V^*$).
\begin{theorem}\label{act-main:thm}
Assume that the optimization problem specified by Equations~\ref{unbal2:eqn},~\ref{sdp-constr:eqn},~\ref{trconstr2:eqn},~\ref{absconstr2:eqn} and~\ref{defmat:eqn} admits a solution with objective value $\B$ and positive definite matrices $U^*, V^*$. Assume further that $\forall i,j\in[n]: \|B_i+B_j\|_{V^*}>0$. Then, for any $\varepsilon\in (0,1]$, Algorithm~\ref{algmain} computes a $\SDP$-unbalanced solution to the unbalanced $2\to 1$ norm problem, such that:\footnote{Given a positive sequence $\{z_i\}_{i\in[n]}$, let $\AM(\{z_i\}_{i\in[n]})$, $\GM(\{z_i\}_{i\in[n]})$, and $\HM(\{z_i\}_{i\in[n]})$ denote its arithmetic, geometric and harmonic means.}
\begin{itemize}
    \item Expected rounded $\ell_1$ objective value of the solution is at least $\Omega(\varepsilon\cdot\OPT)$,\footnote{The constant in the $\Omega$ notation is an absolute constant (see Theorem~\ref{thm:round}).} and
    \item The unbalance factor (measured by Equation~\ref{unbal2:eqn}) is at least:\footnote{Conversely, in the easier direction, the unbalance factor is at most: $\SDP\le 2\cdot\B$ (see Theorem~\ref{thm:main1}).}
    \begin{equation}
        \SDP\ge\frac{\GM\left(\{\|B_i+B_j\|_{V^*}^4\}_{i,j\in[n]}\right)^\frac{1}{2}}{\AM\left(\{\|B_i+B_j\|_{V^*}^4\}_{i,j\in[n]}\right)^\frac{1}{2}}\cdot\frac{\HM\left(\{\|B_i\|_{V^*}^2\|B_j\|_{V^*}^2\}_{i,j\in[n]}\right)^\frac{1}{2}}{\AM\left(\{\|B_i\|_{V^*}^2\|B_j\|_{V^*}^2\}_{i,j\in[n]}\right)^\frac{1}{2}}\cdot\B.\label{eqn:mainthm}
    \end{equation}
\end{itemize}
\end{theorem}

\paragraph{Proof structure} The proof of Theorem~\ref{main:thm} requires the following sub-theorems and lemmas.
Theorem~\ref{main:thm} contains two guarantees:
\begin{enumerate}
\item A guarantee on the unbalance factor as measured by the log-mutual-coherence.\footnote{The objective in Equation~\ref{unbal2:eqn} is essentially mutual coherence and $\mathrm{SDP}(\varepsilon,\xi)$ optimizes its logarithm.}
\item A guarantee on the expected $\ell_1$ objective of the rounded solution.
\end{enumerate}
The second follows from Theorem~\ref{thm:round}; the first from Theorem~\ref{thm:main1}, which relies on Lemma~\ref{lm:main}.

\begin{theorem}\label{thm:main1}
Let $\SDP$ be the solution from Algorithm~\ref{algmain}, and let $\rho(X):=\frac{\AM\left(\{\|B_i+B_j\|_{X}^4\}_{i,j\in[n]}\right)^\frac{1}{2}}{\GM\left(\{\|B_i+B_j\|_{X}^4\}_{i,j\in[n]}\right)^\frac{1}{2}}\cdot\frac{\AM\left(\{\|B_i\|_{X}^2\|B_j\|_{X}^2\}_{i,j\in[n]}\right)^\frac{1}{2}}{\HM\left(\{\|B_i\|_{X}^2\|B_j\|_{X}^2\}_{i,j\in[n]}\right)^\frac{1}{2}}$, then:
\begin{equation}
    \rho(V^*)\cdot\SDP\ge \B \ge \frac{\SDP}{2},
\end{equation}
where $V^*$ is the optimal solution in Theorem~\ref{main:thm}.
\end{theorem}
\noindent\textit{Proof: see Supplementary Section~\ref{sbs:det-proofs}.}

\begin{lemma}\label{lm:main}
Given an $n\times n$ positive definite matrix $X$ and vectors $\{u_i\}_{i\in[n]}$:
\begin{enumerate}
\item \begin{equation}
\frac{2}{n^2}\cdot\sum_{i,j\in[n]}\left(1+\frac{\langle u_i, u_j\rangle_X}{\|u_i\|_X\cdot\|u_j\|_X}\right)\ge\prod_{i,j\in[n]}\left(1+\frac{2\cdot\langle u_i, u_j\rangle_X}{\|u_i\|_X^2+\|u_j\|_X^2}\right)^{\frac{1}{n^2}}.\label{eqn:lb-lm}
\end{equation}
\item Furthermore, assume that $\forall i,j\in[n]:\ \|u_i+u_j\|_X>0$, and let $\rho(X):=\frac{\AM\left(\{\|B_i+B_j\|_{X}^4\}_{i,j\in[n]}\right)^\frac{1}{2}}{\GM\left(\{\|B_i+B_j\|_{X}^4\}_{i,j\in[n]}\right)^\frac{1}{2}}\cdot\frac{\AM\left(\{\|B_i\|_{X}^2\|B_j\|_{X}^2\}_{i,j\in[n]}\right)^\frac{1}{2}}{\HM\left(\{\|B_i\|_{X}^2\|B_j\|_{X}^2\}_{i,j\in[n]}\right)^\frac{1}{2}}$, then:
\begin{equation}
\rho(X)\cdot\prod_{i,j\in[n]}\left(1+\frac{2\cdot\langle u_i, u_j\rangle_X}{\|u_i\|_X^2+\|u_j\|_X^2}\right)^{\frac{1}{n^2}}\ge\frac{1}{n^2}\cdot\sum_{i,j\in[n]}\left(1+\frac{\langle u_i, u_j\rangle_X}{\|u_i\|_X\cdot\|u_j\|_X}\right).\label{eqn:ub-lm}
\end{equation}
\end{enumerate}
\end{lemma}
\noindent\textit{Proof: see Supplementary Section~\ref{sbs:det-proofs}.}

\begin{theorem}\label{thm:round}
The rounded solution from Algorithm~\ref{alground} has a constant approximation factor of $\sqrt{\frac{\pi}{2}}$ for the $\ell_1$ objective relative to the SDP solution.
\end{theorem}
\noindent\textit{Proof: see Supplementary Section~\ref{sbs:det-proofs}.}

\noindent The $\sqrt{\frac{\pi}{2}}$ factor comes from calculations from Grothendieck's identity:
\begin{lemma}\label{lm:grot2}
Given a $2$-dimensional Gaussian $(X,Y)$ with mean zero and covariance matrix $\begin{bmatrix} \delta_x & \rho \\ \rho & \delta_y \end{bmatrix}$:
    \begin{equation}
        \mathbb{E}[\step(X)\step(Y)] = \frac{2}{\pi}\arcsin\left(\frac{\rho}{\sqrt{\delta_x\delta_y}}\right).
    \end{equation}
\end{lemma}
\noindent\textit{Proof: see Supplementary Section~\ref{sbs:det-proofs}.}

\clearpage
\subsection{Supplementary Detailed Proofs}\label{sbs:det-proofs}

\subsubsection*{Proof of Theorem~\ref{thm:var-char} (variational characterization)}\label{sec:proof-varchar}
\begin{proof}
We asymptotically evaluate $\Psi(m)$. Since $\Psi$ is a sum of exponentials, it is dominated by the largest term (Laplace's principle). To obtain simplified expressions in terms of the $\ell_1$ and $\ell_2$ norms, we eliminate lower-order terms by taking limits.

We write $\exp\{\Psi\}$ and simplify, tracking linear boundary terms (we eventually let $B\to 0$ and differentiate to obtain the magnetization). Let $H(x)$ denote the Hamiltonian and $H$ the $m\times n$ matrix with column $H_{i}\equiv (h_{1i},...,h_{mi})$, so that $H(x)\equiv xHH^Tx$. Let $Z$ denote a standard $m$-dimensional Gaussian:
\begin{eqnarray}
    \exp\{\Psi\}&:=&\sum_{x\in\{\pm1\}^n}\exp\left(\frac{\beta}{n}H(x)+B\left(\sum_{i\in[n]}x_i\right)\right)\\
    &=&\sum_{x\in\{\pm1\}^n}\EE\left[\exp\left(\frac{\beta}{\sqrt{n}}xHZ\right)\right]\exp\left(\beta B\left(\sum_{i\in[n]}x_i\right)\right)\\
    &=&2^n\EE\left[\beta \prod_{i=1}^{n}\cosh{\left(\sum_{j=1}^mH_{ji}Z_j+B\right)}\right]\\
    &=&2^n\EE\left[\exp\left(\beta\sum_{i=1}^{n}\log\cosh{\left(\sum_{j=1}^mH_{ji}Z_j+B\right)}\right)\right]
\end{eqnarray}
To evaluate the sum in the exponential, group the $n$ terms with the same tuple $H_i:=(A_{i1},...,A_{im})$; there are approximately $n_i=w\cdot n$ such terms, where $w=\frac{m}{n}\pm o(1)$.

By the law of large numbers, each group converges to its mean. For group $k\in[m]$ with $n_k$ terms:
\begin{eqnarray}
    S_k&:=&\sum_{k\in K}\log\cosh\left(\sum_{j=1}^{m}A_{kj}Z_j+B\right)\\
    &\simeq&n_k\int_\RR\log\cosh\left(\sum_{j=1}^{m}A_{kj}Z_j+B\right)\mathrm{d}p_K(\vec{H}_k),
\end{eqnarray}
where $\vec{L}_k$ identifies group $k$. Substituting $Z_j=\frac{n}{n_k}\zeta_j$ and $B=\frac{n}{n_k}B_k$:
\begin{eqnarray}
    S_k&\simeq&n_k\cdot\log\cosh\left(\frac{n}{n_k}\left(\sum_{j=1}^{m}A_{kj}\zeta_j+B_k\right)\right)\\
    &\simeq&n\cdot\frac{\log\cosh\left(\frac{n}{n_K}\left(\sum_{j=1}^{m}A_{kj}\zeta_j+B_k\right)\right)}{\frac{n}{n_k}\left(\sum_{j=1}^{m}A_{kj}\zeta_j+B_k\right)}\left(\sum_{j=1}^{m}A_{kj}\zeta_j+B_k\right)
\end{eqnarray}
For $0\ll n_K\ll n$, L'H\^{o}pital's rule gives:
\begin{equation}\label{eqn:detcase}
    S_k=n\cdot\left|\sum_{j=1}^{m}A_{kj}\zeta_j+B_k\right|.
\end{equation}

For large $m$, $n$, $\frac{m}{n}\to 0$ and $B\to 0^+$, Laplace's principle yields:
\begin{equation}
     \Psi \to \max_{\zeta\in\RR^m}\left(\beta\sum_{k\in[m]}\left|\sum_{j\in[m]}A_{kj}\zeta_j\right|-\frac{1}{2}\cdot\zeta^T\zeta\right).\label{eqn:interim-max}
\end{equation}
This is equivalent to computing the $2\to 1$ norm with constraint $\|\zeta\|\le y$.

The magnetization equals the average partial derivative of $\Psi(m,B)$ with respect to $B_k$ (as $B_k\to 0^+$). From Equation~\ref{eqn:detcase}:
\begin{equation}
    \Phi = \frac{1}{m}\sum_{k\in[m]}\step\left(\beta\sum_{j\in[m]}A_{kj}\zeta_j^*\right)=\frac{1}{m}\sum_{k\in[m]}\step\left(\sum_{j\in[m]}A_{kj}\zeta_j^*\right).
\end{equation}
\end{proof}

\subsubsection*{Proof of Theorem~\ref{main:thm} (unbalance factor bounds)}\label{sec:proof-main}
\begin{proof}
Let $\widetilde{V}$ denote the optimal solution from Algorithm~\ref{algmain}. Setting $X = \widetilde{V}$ in Equation~\ref{eqn:lb-lm}:
\begin{equation}
    \frac{2}{n^2}\cdot\sum_{i,j\in[n]}\left(1+\frac{\langle u_i, u_j\rangle_{\widetilde{V}}}{\|u_i\|_{\widetilde{V}}\cdot\|u_j\|_{\widetilde{V}}}\right)\ge\prod_{i,j\in[n]}\left(1+\frac{2\cdot\langle u_i, u_j\rangle_{\widetilde{V}}}{\|u_i\|_{\widetilde{V}}^2+\|u_j\|_{\widetilde{V}}^2}\right)^{\frac{1}{n^2}}.\label{eqn:step1}
\end{equation}
The RHS is at least $2\cdot\SDP$ (by AM-GM on the denominator); the LHS is at most $4\cdot\B$. Hence $2\cdot\B\ge\SDP$.

For the upper bound, set $X=V^*$ in Equation~\ref{eqn:ub-lm}:
\begin{equation}
\rho(V^*)\cdot\frac{\prod_{i,j\in[n]}\left(\|B_i+B_j\|_{V^*}^2\right)^\frac{1}{n^2}}{\frac{4}{n}\cdot\sum_{i\in[n]}\|B_i\|_{V^*}^2}\ge\frac{1}{n^2}\cdot\sum_{i,j\in[n]}\left(1+\frac{\langle u_i, u_j\rangle_{V^*}}{\|u_i\|_{V^*}\cdot\|u_j\|_{V^*}}\right).\label{eqn:ub-step2}
\end{equation}
Since Algorithm~\ref{algmain}'s optimum ($\SDP$) dominates the geometric mean:
\begin{equation}
\frac{\prod_{i,j\in[n]}\left(\|B_i+B_j\|_{\widetilde{V}}^2\right)^\frac{1}{n^2}}{\frac{4}{n}\cdot\sum_{i\in[n]}\|B_i\|_{\widetilde{V}}^2}
\ge\frac{\prod_{i,j\in[n]}\left(\|B_i+B_j\|_{V^*}^2\right)^\frac{1}{n^2}}{\frac{4}{n}\cdot\sum_{i\in[n]}\|B_i\|_{V^*}^2},\label{eqn:ub-step3}
\end{equation}
and substituting into Equation~\ref{eqn:ub-step2} gives the upper bound on $\B$.
\end{proof}

\subsubsection*{Proof of Theorem~\ref{thm:polytime} (polynomial time)}\label{sec:proof-polytime}
\begin{proof}
$\mathrm{SDP}(\varepsilon,\xi)$ is feasible for $\varepsilon < 1$ and $\xi\ge\frac{4}{n}\cdot\sum_{i}\|B_i\|_{V^*}^2$. Since the maximum eigenvalue of $V^*$ is at most $n$ and $\|B_i\|\le n\cdot\max_{i,j}|B_{i,j}|$, we have $\frac{1}{n}\cdot\sum_{i}\|B_i\|_{V^*}^2\le n^3\cdot (\max_{i,j}|B_{i,j}|)^2$. Hence Algorithm~\ref{algmain} encounters at least one feasible $\xi$.

If $\xi^*\in(h,n^3\cdot(\max_{i,j}|B_{ij}|)^2)$, then any returned $\widetilde{\xi}$ is within a $(1+\frac{1}{n})$ factor of $\xi^*$. If $0<\xi^*<h$, then $\frac{1}{n}\cdot\sum_{i}\|B_i\|_{V^*}^2 < \xi^*$, implying some $\|B_i\|_{V^*}^2\le\xi^*$, so the minimum eigenvalue of $V^*$ is less than $h$, a contradiction. Algorithm~\ref{algmain} makes polynomially many calls to Algorithm~\ref{alground}, which runs in polynomial time.
\end{proof}

\subsubsection*{Proof of Theorem~\ref{thm:round} ($\ell_1$ approximation guarantee)}
\begin{proof}
We compute $\EE\left[\langle v_i,g\rangle\cdot\step(\langle u_j,g\rangle)\right]$ in terms of $\langle v_i, u_j\rangle$, using the fact that for Gaussians $\gamma_1,\gamma_2$ with variances $\sigma_1,\sigma_2$ and correlation $\varrho$: $\EE[\gamma_1\cdot\step(\gamma_2)] =\varrho\sigma_1\sigma_2\sqrt{\frac{2}{\pi}}$.\footnote{This analysis parallels the proof via Reitz' method in~\cite{alon-naor}.}

For $g=(g(1),...,g(n))$ in Algorithm~\ref{alground}, define for each $k\in[n]$:
\begin{equation}
h_{j}(k):=u_j(k)g(k)+g_k^\perp,\label{eqn:gijk}
\end{equation}
where $g_k^\perp$ is uncorrelated with $g(k)$ and $\langle u_j,g\rangle=h_j$. Then:
\begin{eqnarray}
\EE\left[\langle v_i,g\rangle\cdot\step(\langle u_j,g\rangle)\right] &=& \EE[\langle v_i,g\rangle\cdot\step(h_j)]\\
&=& \EE\left[\sum_{k\in[n]} v_i(k)g(k)\cdot\step(h_j(k))\right]\\
&=& \sqrt{\frac{2}{\pi}}\cdot\sum_{k\in[n]} v_i(k)u_j(k)\\
&=& \sqrt{\frac{2}{\pi}}\cdot\langle v_i,u_j\rangle.
\end{eqnarray}
Summing over all $i,j\in[n]$ gives the $\ell_1$ objective in terms of the SDP solution.
\end{proof}

\subsubsection*{Proof of Lemma~\ref{lm:grot2} (Grothendieck's identity)}
\begin{proof}
Let $Z:=\frac{X}{\sqrt{\delta_x}}-\frac{\rho Y}{\sqrt{\delta_x}\delta_y}$, a mean-zero Gaussian independent of $Y$ with variance $1+\frac{\rho^2}{\delta_x\delta_y}$. Then:
\begin{eqnarray}
    \PR\left(XY\ge 0\right) &=& \PR\left(ZY\ge \frac{\rho}{\sqrt{\delta_x}\delta_y}Y^2\right)\\
    &=& \PR\left(\frac{Z}{\sqrt{1+\frac{\rho^2}{\delta_x\delta_y}}}\ge\frac{\rho Y}{\sqrt{\delta_x}\delta_y\sqrt{1+\frac{\rho^2}{\delta_x\delta_y}}}\right)\\
    &=&\PR\left(\frac{\tilde{Z}}{\tilde{Y}}\ge\frac{\rho}{\sqrt{\delta_x\delta_y}\sqrt{1+\frac{\rho^2}{\delta_x\delta_y}}}\right),
\end{eqnarray}
where $\tilde{Y},\tilde{Z}$ are independent standard Gaussians, so their ratio follows the Cauchy distribution:
\begin{eqnarray}
    \PR\left(XY\ge 0\right) &=& \frac{1}{\pi}\arctan\left(\frac{\rho}{\sqrt{\delta_x\delta_y}\sqrt{1+\frac{\rho^2}{\delta_x\delta_y}}}\right)\\
    &=& \frac{1}{\pi}\arcsin\left(\frac{\rho}{\sqrt{\delta_x\delta_y}}\right).
\end{eqnarray}
\end{proof}

\subsubsection*{Proof of Lemma~\ref{lm:main} (core inequality)}
\begin{proof}
The proof proceeds in three steps.

\textbf{Step 1.} Show:
    \begin{equation}
      \frac{1}{n^2}\cdot\sum_{i,j\in[n]}\left(1+\frac{\langle u_i, u_j\rangle_X}{\|u_i\|_X\cdot\|u_j\|_X}\right)\ge\frac{1}{2n^2}\cdot\sum_{i,j\in[n]}\left(1+\frac{2\cdot\langle u_i, u_j\rangle_X}{\|u_i\|_X^2+\|u_j\|_X^2}\right).\label{eqn:lb1}
    \end{equation}

The RHS of~\eqref{eqn:lb1} is at most $1$ (each summand is in $[0,2]$). The LHS can be rewritten as:
\begin{eqnarray}
2 + \frac{2}{n^2}\cdot\sum_{i,j\in[n]}\frac{\langle u_i, u_j\rangle_X}{\|u_i\|_X\cdot\|u_j\|_X} &=& 2+ \frac{2}{n^2}\cdot\sum_{i,j\in[n]}\langle \frac{u_i}{\|u_i\|_X}, \frac{u_j}{\|u_j\|_X}\rangle_X\\
&=& 2 + \frac{2}{n^2}\cdot\langle \sum_{i\in[n]}\frac{u_i}{\|u_i\|_X}, \sum_{j\in[n]}\frac{u_j}{\|u_j\|_X}\rangle_X\\
&\ge& 2,
\end{eqnarray}
by non-negativity of the norm under positive semidefinite $X$. The AM-GM inequality on the RHS of~\eqref{eqn:lb1} yields~\eqref{eqn:lb-lm}.

\textbf{Step 2.} Show:
    \begin{equation}
        \frac{1}{n^2}\cdot\sum_{i,j\in[n]}\frac{\|u_i+u_j\|_X^2}{\|u_i\|_X\cdot\|u_j\|_X}\ge \frac{2}{n^2}\cdot\sum_{i,j\in[n]}\left(1+\frac{\langle u_i, u_j\rangle_X}{\|u_i\|_X\cdot\|u_j\|_X}\right).\nonumber
    \end{equation}
This follows from:
\begin{equation}
\|u_i+u_j\|_X^2 \ge 2\cdot\left(\frac{\|u_i\|_X^2+\|u_j\|_X^2}{2}+\langle u_i,u_j\rangle_X\right)\ge 2\cdot\left(\|u_i\|_X\cdot\|u_j\|_X+\langle u_i,u_j\rangle_X\right),
\end{equation}
where the last inequality is AM-GM.

\textbf{Step 3.} Show:
    \begin{equation}
       \rho(X)\cdot\frac{\prod_{i,j\in[n]}\left(\|u_i+u_j\|_{X}^2\right)^\frac{1}{n^2}}{\frac{4}{n}\cdot\sum_{i\in[n]}\|u_i\|_{X}^2}\ge \frac{1}{4n^2}\cdot\sum_{i,j\in[n]}\frac{\|u_i+u_j\|_X^2}{\|u_i\|_X\cdot\|u_j\|_X}.\nonumber
    \end{equation}

We need:
\begin{equation}
    \rho(X) \ge \frac{\frac{1}{4n^2}\cdot\sum_{i,j\in[n]}\frac{\|u_i+u_j\|_X^2}{\|u_i\|_X\cdot\|u_j\|_X}}{\frac{\prod_{i,j\in[n]}\left(\|u_i+u_j\|_{X}^2\right)^\frac{1}{n^2}}{\frac{4}{n}\cdot\sum_{i\in[n]}\|u_i\|_{X}^2}}.\label{step3:eqn}
\end{equation}
By Cauchy--Schwarz:
\begin{eqnarray}
    \frac{\frac{1}{4n^2}\cdot\sum_{i,j}\frac{\|u_i+u_j\|_X^2}{\|u_i\|_X\cdot\|u_j\|_X}}{\frac{\prod_{i,j}\left(\|u_i+u_j\|_{X}^2\right)^\frac{1}{n^2}}{\frac{4}{n}\cdot\sum_{i}\|u_i\|_{X}^2}}
    \le\frac{\left(\frac{1}{n^2}\cdot\sum_{i,j}\|u_i+u_j\|_X^4\right)^\frac{1}{2}\left(\frac{1}{n^2}\cdot\sum_{i,j}\frac{1}{\|u_i\|_X^2\cdot\|u_j\|_X^2}\right)^\frac{1}{2}}{\prod_{i,j}\left(\|u_i\|_X+\|u_j\|_X\right)^{\frac{2}{n^2}}}\cdot\frac{1}{n}\cdot\sum_{i}\|u_i\|_{X}^2.\nonumber
\end{eqnarray}
Using:
\begin{equation}
    \frac{1}{n}\cdot\sum_{i\in[n]}\|u_i\|_X^2 = \left(\frac{1}{n^2}\cdot\sum_{i,j\in[n]}\|u_i\|_X^2\|u_j\|_X^2\right)^\frac{1}{2},\label{am:eqn}
\end{equation}
the upper bound for the RHS of~\eqref{step3:eqn} becomes:
\begin{equation}
     \frac{\left(\frac{1}{n^2}\cdot\sum_{i,j}\|u_i+u_j\|_X^4\right)^\frac{1}{2}}{\prod_{i,j}\left(\|u_i\|_X+\|u_j\|_X\right)^{\frac{2}{n^2}}}\cdot\frac{ \left(\frac{1}{n^2}\cdot\sum_{i,j}\|u_i\|_X^2\|u_j\|_X^2\right)^\frac{1}{2}}{\left(\frac{n^2}{\sum_{i,j}\frac{1}{\|u_i\|_X^2\cdot\|u_j\|_X^2}}\right)^{\frac{1}{2}}}.\label{fin:eqn}
\end{equation}
Expression~\eqref{fin:eqn} equals $\rho(X)$. Since $X$ is positive definite and $\|u_i+u_j\|_X > 0$, the denominators are positive.
\end{proof}

\clearpage
\subsection{Supplementary Algorithmic Validation}\label{sec:supp-algval}

We validated Algorithm~\ref{algmain} on synthetic $\ell_1$-PCA instances with $B = I_{10}$ and $A$ drawn from the $10 \times 10$ Gaussian Orthogonal Ensemble. The unconstrained SDP baseline produces approximately balanced solutions (Supplementary Fig.~\ref{fig:l1pca_quad}a). Introducing the conditioning constraint at $\varepsilon = 0.5$ shifts the solutions toward greater unbalance (Supplementary Fig.~\ref{fig:l1pca_quad}b). The objective ratio varies linearly with $\varepsilon$ (Supplementary Fig.~\ref{fig:l1pca_quad}c), confirming that the SDP constraint remains tight, and the relative unbalance factor approaches its theoretical maximum of $2$ at small $\varepsilon$ (Supplementary Fig.~\ref{fig:l1pca_quad}d).

\begin{figure}[htbp]
  \centering
  \begin{subfigure}[b]{0.45\linewidth}
    \centering
    \includegraphics[width=.85\linewidth]{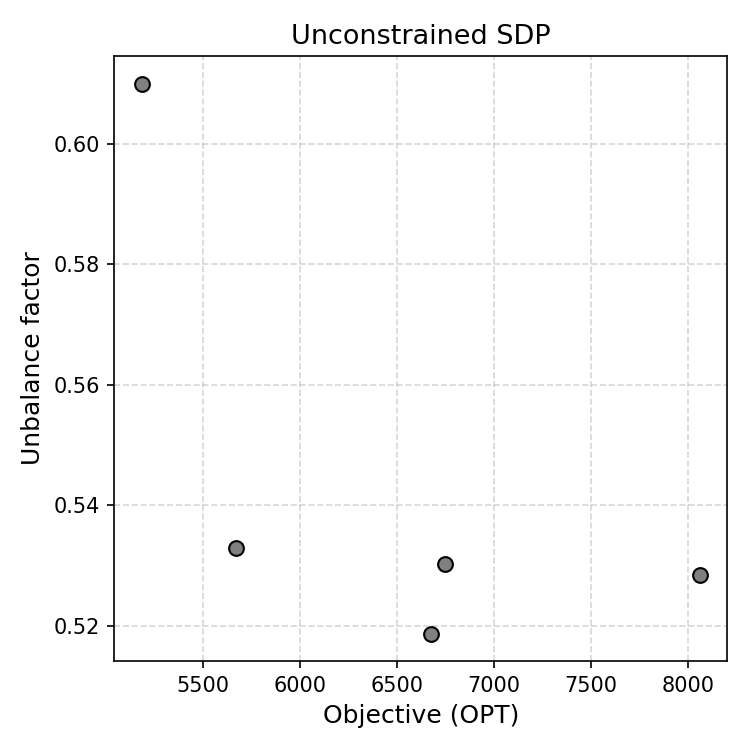}
    \caption{{\footnotesize Unconstrained SDP: objective vs.\ unbalance factor.}}\label{fig:l1pca1}
  \end{subfigure}\hfill
  \begin{subfigure}[b]{0.45\linewidth}
    \centering
    \includegraphics[width=.86\linewidth]{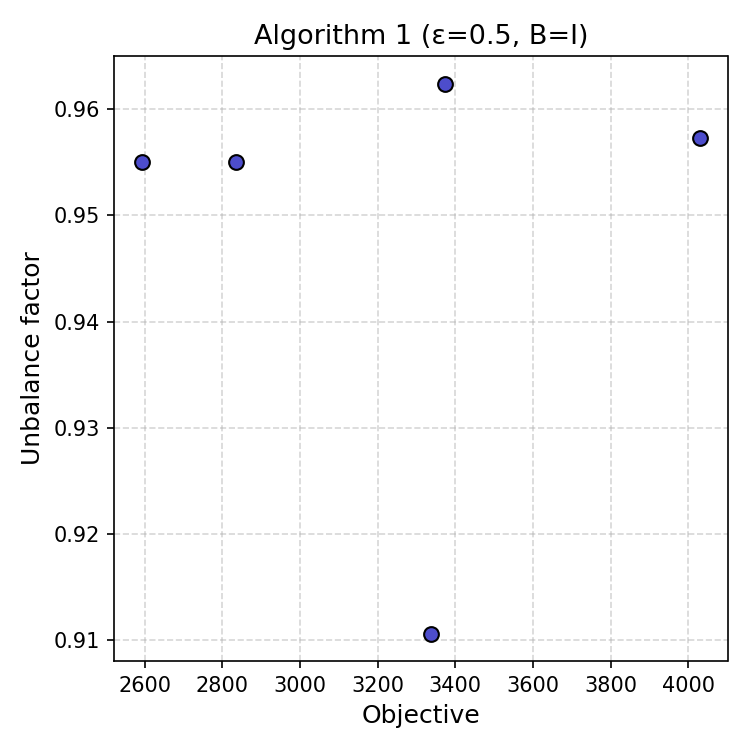}
    \caption{{\footnotesize Algorithm~\ref{algmain} with $\varepsilon=0.5$: objective vs.\ unbalance factor.}}\label{fig:l1pca2}
  \end{subfigure}

  \vspace{1ex}

  \begin{subfigure}[b]{0.45\linewidth}
    \centering
    \includegraphics[width=.85\linewidth]{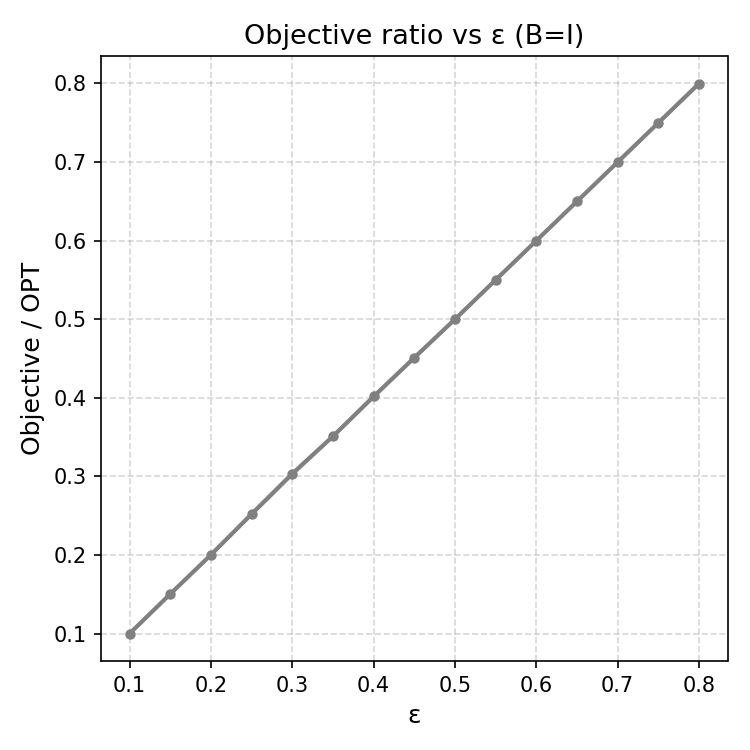}
    \caption{{\footnotesize Objective vs.\ $\varepsilon$.}}\label{fig:l1pca3}
  \end{subfigure}\hfill
  \begin{subfigure}[b]{0.45\linewidth}
    \centering
    \includegraphics[width=1\linewidth]{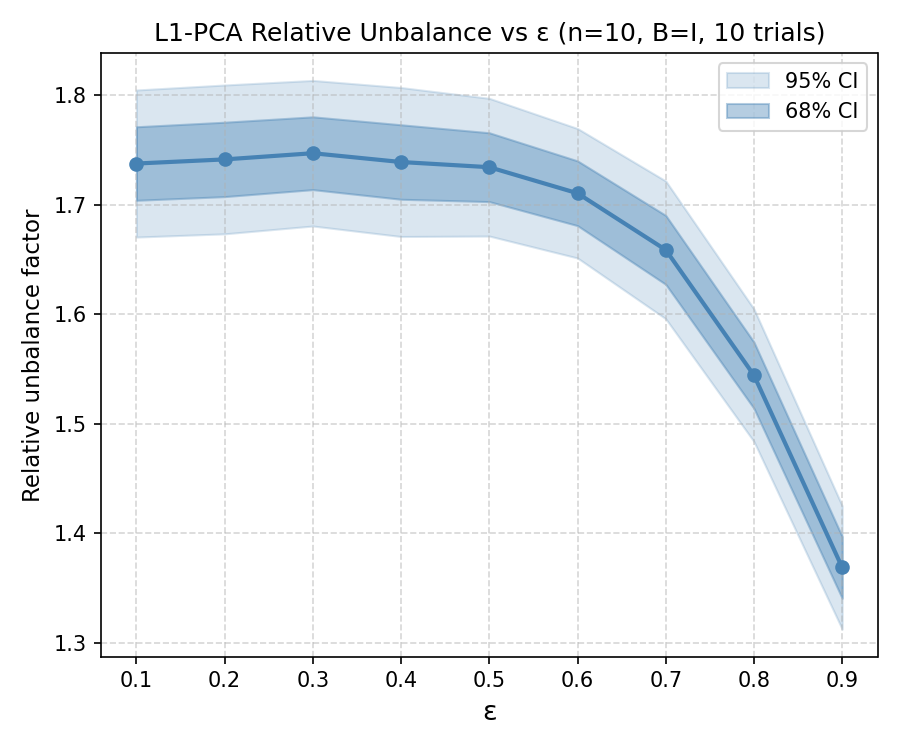}
    \caption{{\footnotesize Relative unbalance factor vs.\ $\varepsilon$.}}\label{fig:l1pca4}
  \end{subfigure}
  \caption{{\footnotesize Synthetic $\ell_1$-PCA validation of Algorithm~\ref{algmain} with $B=I_{10}$ and $A$ drawn from the $10\times 10$ GOE. (a)~Unconstrained baseline is approximately balanced. (b)~$\varepsilon=0.5$ yields more unbalanced solutions. (c)~Objective ratio is approximately linear in $\varepsilon$. (d)~Relative unbalance is nearly $2$ (its maximum) for small $\eps$.}}\label{fig:l1pca_quad}
\end{figure}

\subsection*{Supplementary Pipeline Validation}\label{sec:supp-validation}

Supplementary Fig.~\ref{fig:pdb-ham} shows the Hamiltonian matrix constructed from 150 PDB-derived J-coupling constants (75 diagonal entries from ${}^3J(\phi)$, 75 off-diagonal entries from ${}^2J(\psi)$).

\begin{figure}[htbp]
\centering
\includegraphics[width=0.45\textwidth]{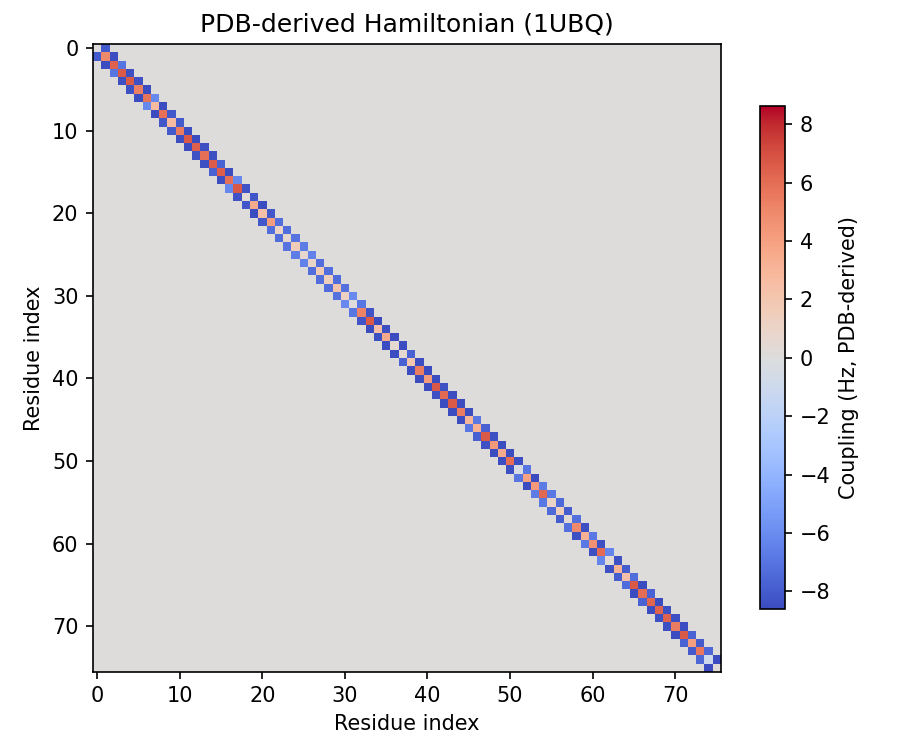}
\caption{Hamiltonian matrix for Ubiquitin, constructed from PDB-derived J-coupling constants. Intra-residue couplings (from $\phi$) contribute to the diagonal; sequential couplings (from $\psi$) contribute to the off-diagonal.}
\label{fig:pdb-ham}
\end{figure}

Supplementary Fig.~\ref{fig:pdb-rama-baseline} shows the Ramachandran plot of the 1UBQ crystal structure, serving as the ground-truth baseline for all conditioned structures in Section~\ref{sec:application}.

\begin{figure}[htbp]
\centering
\includegraphics[width=0.55\textwidth]{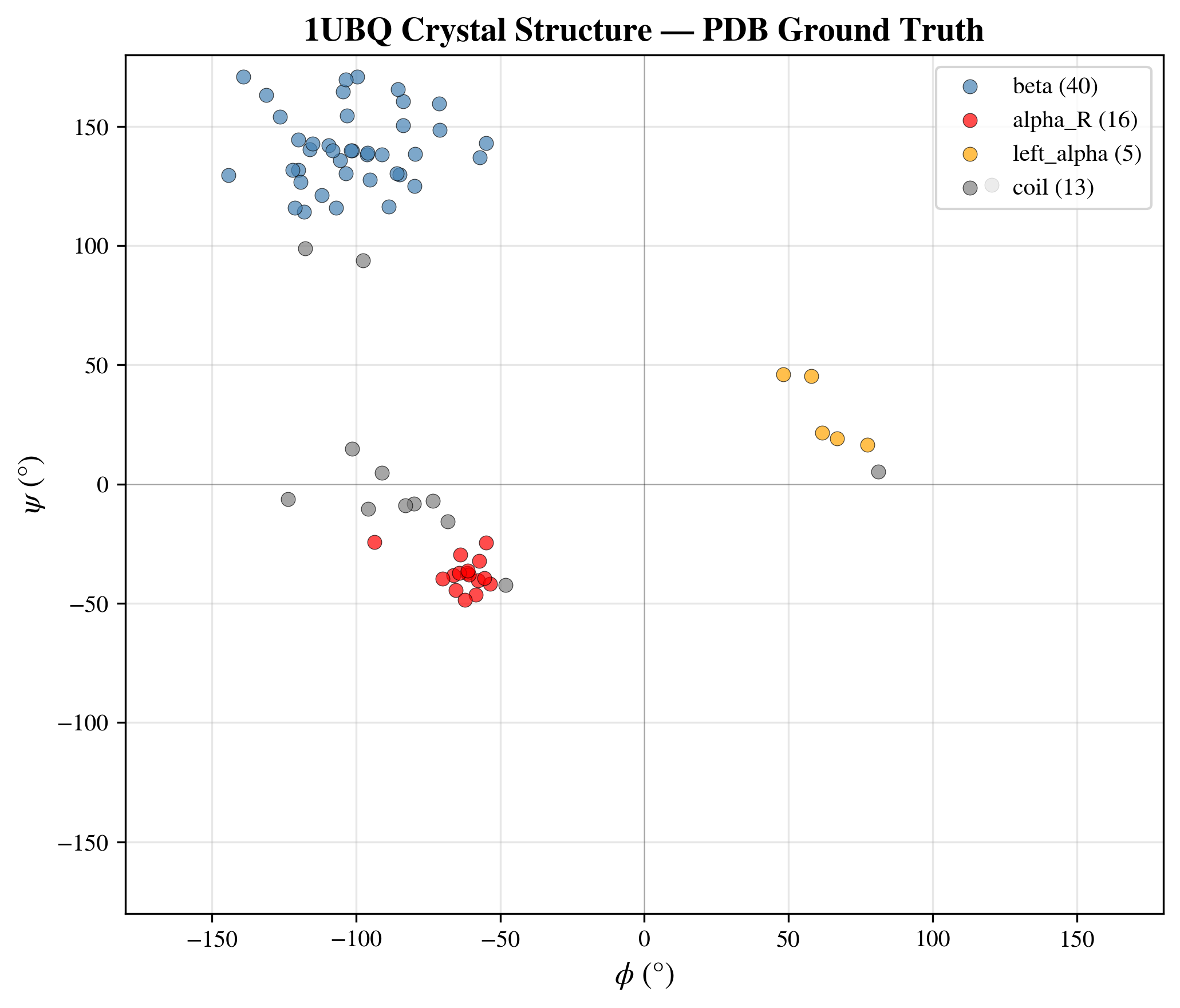}
\caption{Ramachandran plot of the 1UBQ crystal structure (PDB ground truth). The 74 backbone residues partition into $\beta$-sheet (blue, 40), $\alpha_R$-helix (red, 16), left-handed $\alpha$ (orange, 5), and coil (gray, 13).}
\label{fig:pdb-rama-baseline}
\end{figure}

To verify the pipeline's self-consistency, we performed a round-trip test: forward Karplus $\to$ Hamiltonian $\to$ SDP ($\varepsilon = 0.9$) $\to$ Karplus$^{-1}$. The recovered $\phi$ angles have a maximum error of $0.05^\circ$ (mean $0.03^\circ$), and $\psi$ angles are recovered exactly (max error $< 10^{-10}$ degrees). Supplementary Fig.~\ref{fig:roundtrip} shows the recovered vs.\ PDB angles and the error distributions.

\begin{figure}[htb!]
\centering
\includegraphics[width=0.95\textwidth]{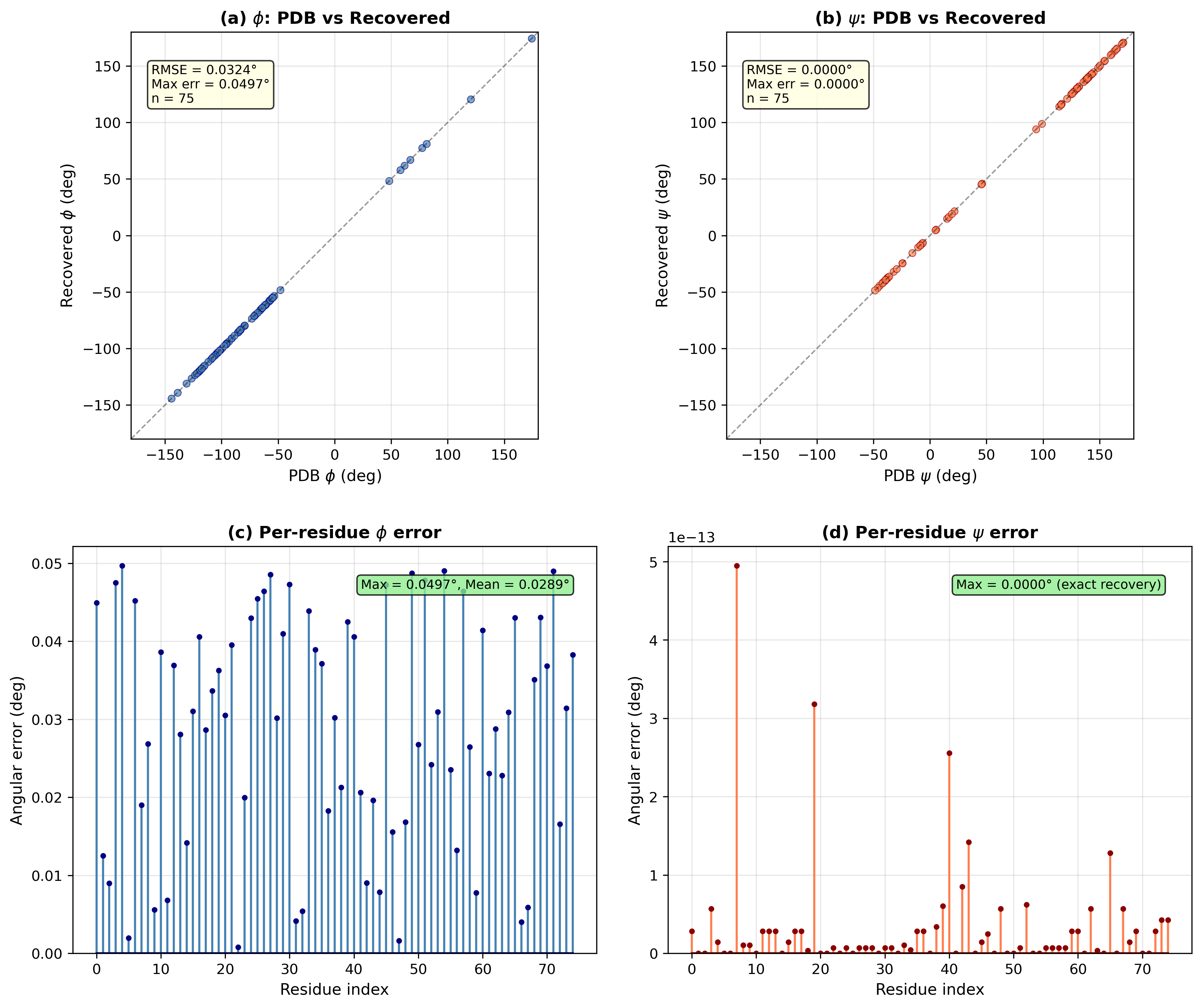}
\caption{Round-trip pipeline validation: forward Karplus $\to$ Hamiltonian $\to$ SDP$(\varepsilon{=}0.9)$ $\to$ Karplus$^{-1}$. (a)~Recovered $\phi$ vs.\ PDB $\phi$ (RMSE $= 0.03^\circ$). (b)~Recovered $\psi$ vs.\ PDB $\psi$ (exact recovery). (c)~Per-residue $\phi$ error (max $0.05^\circ$). (d)~Per-residue $\psi$ error (all below $10^{-10}$ degrees).}
\label{fig:roundtrip}
\end{figure}

Supplementary Fig.~\ref{fig:rmsd-vs-eps} shows the angular RMSD between SDP-recovered dihedral angles and the crystal structure across the $\varepsilon$ sweep.

\begin{figure}[htbp]
\centering
\includegraphics[width=0.6\textwidth]{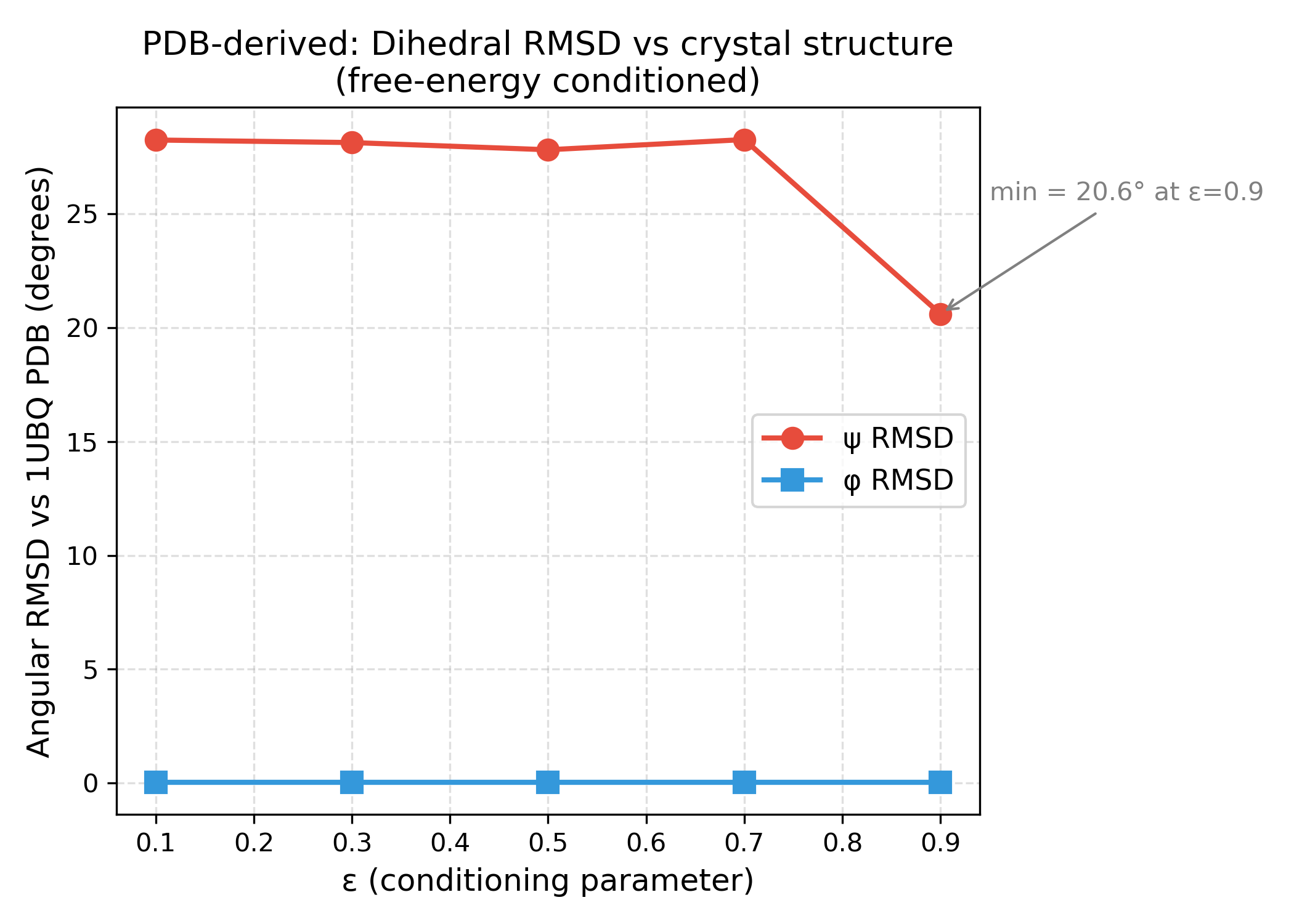}
\caption{Angular RMSD (degrees) between SDP-recovered dihedral angles and the 1UBQ crystal structure, as a function of $\varepsilon$. The $\psi$ RMSD is approximately $28^\circ$ for $\varepsilon \le 0.7$ and decreases to $20.6^\circ$ at $\varepsilon = 0.9$. The $\phi$ RMSD is identically zero.}
\label{fig:rmsd-vs-eps}
\end{figure}

\subsection{Supplementary Blending Weight Sensitivity}\label{sec:alpha-sensitivity}

The blending weight $\alpha$ controls the magnitude of perturbation at sign-flipped bonds. To verify robustness, we repeated the $\varepsilon$-sweep at $\alpha = 0.7$ (effective coupling $0.4\,J$) and $\alpha = 0.8$ (effective coupling $0.6\,J$).

At the coarse level (helix vs.\ $\beta$-sheet vs.\ coil), the three $\alpha$ values produce nearly identical results: pairwise agreement exceeds 95\% at every $\varepsilon$, and $\beta$-sheet remains invariant at 42 residues.

At the fine-grained level, reducing $\alpha$ causes a systematic $\alpha_R \to \pi$-helix reclassification affecting 12--16 residues at $\alpha = 0.7$ (versus 0 $\pi$-helix residues at $\alpha = 0.9$). The affected positions are concentrated in the main $\alpha$-helix (residues 23--30) and at scattered helical sites. This shift is expected: $\alpha_R$ and $\pi$-helix share the same $\phi$ range and differ only in $\psi$ ($[-60^\circ, -30^\circ]$ vs.\ $[-80^\circ, -60^\circ]$), so a modest $\psi$ shift at lower $\alpha$ suffices to cross the boundary. No residues cross the helix/$\beta$/coil boundary at any $\alpha$ value.

\begin{figure}[h!]
\centering
\includegraphics[width=\textwidth]{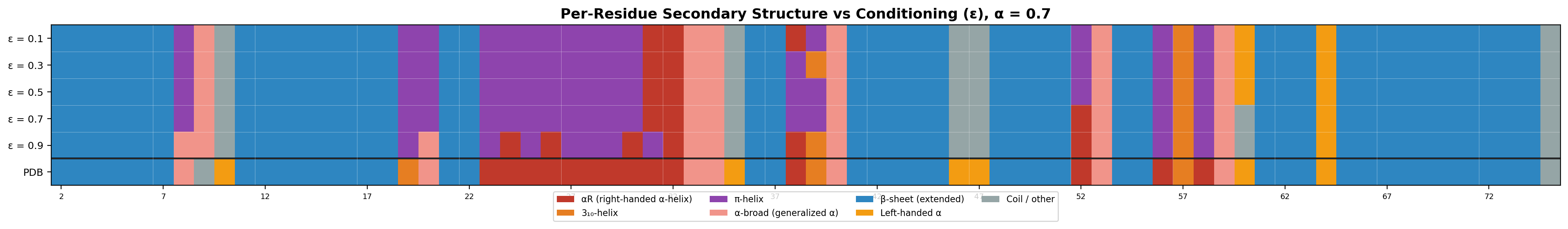}
\caption{Per-residue secondary structure vs.\ $\varepsilon$ at $\alpha = 0.7$. Compare with Fig.~\ref{fig:ss-strip} ($\alpha = 0.9$).}
\label{fig:alpha-sensitivity-07}
\end{figure}

\begin{figure}[h!]
\centering
\includegraphics[width=\textwidth]{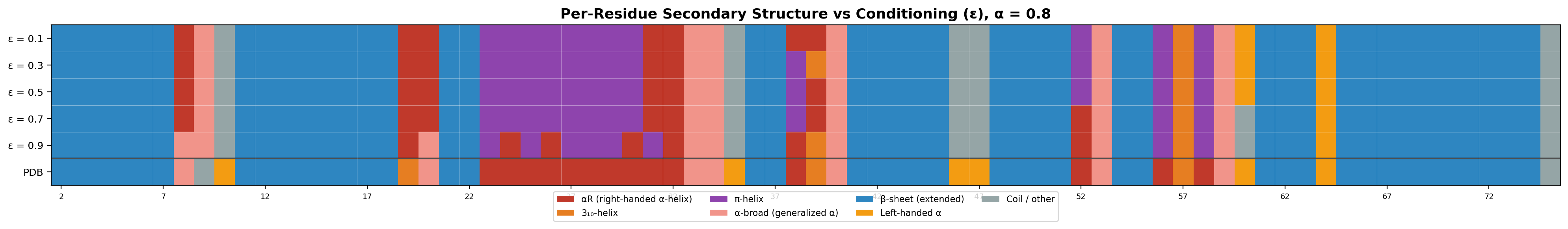}
\caption{Per-residue secondary structure vs.\ $\varepsilon$ at $\alpha = 0.8$. Compare with Fig.~\ref{fig:ss-strip} ($\alpha = 0.9$).}
\label{fig:alpha-sensitivity-08}
\end{figure}

\end{document}